\documentclass[10pt,twocolumn,letterpaper]{article}

\usepackage[pagenumbers]{cvpr}      
\usepackage{subfiles} 

\newcommand{\Prob}{\mathbb{P}}

\newcommand{\LF}{\mathcal{L}}

\newcommand{\R}{\mathbb{R}}

\newcommand{\E}{\mathbb{E}}

\usepackage{amsmath,amssymb,mathtools}
\usepackage{amsthm}
\usepackage{algorithm}
\usepackage{algpseudocode}
\usepackage{booktabs}
\usepackage{graphicx}
\usepackage{placeins}
\usepackage{dblfloatfix}
\usepackage{nicefrac}
\usepackage{multirow}
\usepackage[usenames,dvipsnames]{xcolor}

\definecolor{cvprblue}{rgb}{0.21,0.49,0.74}

\usepackage[pagebackref,breaklinks,colorlinks,allcolors=cvprblue]{hyperref}

\usepackage{aliascnt}

\newtheorem{theorem}{Theorem}[section]

\newaliascnt{lemma}{theorem}
\newtheorem{lemma}[lemma]{Lemma}
\aliascntresetthe{lemma}

\newaliascnt{proposition}{theorem}
\newtheorem{proposition}[proposition]{Proposition}
\aliascntresetthe{proposition}

\newaliascnt{corollary}{theorem}
\newtheorem{corollary}[corollary]{Corollary}
\aliascntresetthe{corollary}

\theoremstyle{definition}
\newaliascnt{definition}{theorem}
\newtheorem{definition}[definition]{Definition}
\aliascntresetthe{definition}

\newaliascnt{assumption}{theorem}
\newtheorem{assumption}[assumption]{Assumption}
\aliascntresetthe{assumption}

\newaliascnt{example}{theorem}

\aliascntresetthe{example}

\theoremstyle{remark}
\newaliascnt{remark}{theorem}
\newtheorem{remark}[remark]{Remark}
\aliascntresetthe{remark}

\newaliascnt{note}{theorem}

\aliascntresetthe{note}

\definecolor{cvprblue}{rgb}{0.21,0.49,0.74}

\usepackage[nameinlink,capitalise,noabbrev]{cleveref}
\crefname{theorem}{theorem}{theorems}      
\Crefname{theorem}{Theorem}{Theorems}
\crefname{lemma}{lemma}{lemmas}            
\Crefname{lemma}{Lemma}{Lemmas}
\crefname{proposition}{proposition}{propositions}
\Crefname{proposition}{Proposition}{Propositions}
\crefname{corollary}{corollary}{corollaries}
\Crefname{corollary}{Corollary}{Corollaries}
\crefname{definition}{definition}{definitions}
\Crefname{definition}{Definition}{Definitions}
\crefname{assumption}{assumption}{assumptions}
\Crefname{assumption}{Assumption}{Assumptions}
\crefname{example}{example}{examples}      \Crefname{example}{Example}{Examples}
\crefname{remark}{remark}{remarks}         \Crefname{remark}{Remark}{Remarks}
\crefname{note}{note}{notes}   
\Crefname{note}{Note}{Notes}
\newcommand{\methodname}{\textbf{RR-DU}}
\def\paperID{*} 
\def\confName{CVPR}
\def\confYear{2026}

\title{Fully Decentralized Certified Unlearning}

\author{Hithem Lamri\\
NYU Abu Dhabi\\
Abu Dhabi\\
{\tt\small hithem.lamri@nyu.edu}
\and
Michail Maniatakos\\
NYU Abu Dhabi\\
Abu Dhabi\\
{\tt\small michail.maniatakos@nyu.edu}
}

\begin{document}

\maketitle
\begin{abstract}
Machine unlearning (MU) seeks to remove the influence of specified data from a trained model in response to privacy requests or data poisoning. While certified unlearning has been analyzed in centralized and server-orchestrated federated settings (via guarantees analogous to differential privacy, DP), the decentralized setting—where peers communicate without a coordinator—remains underexplored. We study certified unlearning in decentralized networks with fixed topologies and propose \methodname, a random-walk procedure that performs one projected gradient ascent step on the forget set at the unlearning client and a geometrically distributed number of projected descent steps on the retained data elsewhere, combined with subsampled Gaussian noise and projection onto a trust region around the original model. We provide (i) convergence guarantees in the convex case and stationarity guarantees in the nonconvex case, (ii) $(\varepsilon,\delta)$ network-unlearning certificates on client views via subsampled Gaussian R\'enyi DP (RDP) with segment-level subsampling, and (iii) deletion-capacity bounds that scale with the forget-to-local data ratio and quantify the effect of decentralization (network mixing and randomized subsampling) on the privacy–utility trade-off. Empirically, on image benchmarks (MNIST, CIFAR-10), \methodname~ matches a given $(\varepsilon,\delta)$ while achieving higher test accuracy than decentralized DP baselines and reducing forget accuracy to random guessing (\(\approx 10\%\)).

\end{abstract}    
\section{Introduction}
Machine Unlearning (MU) aims to remove the influence of a designated subset—the unlearning set—from a trained model while preserving performance on the retained data. As ML is deployed in healthcare, finance, and vision, regulations such as GDPR~\cite{gdpr} and CCPA~\cite{ccpa} formalize a \emph{right to be forgotten}, and the growing risk of data poisoning further motivates MU. Simply deleting records is insufficient—their effect persists in the model parameters. The naive remedy, retraining from scratch without the unlearning set, is often impractical for large models and production systems. This raises the core challenge: efficiently erasing a subset’s influence while maintaining utility.

Machine Unlearning was first introduced by~\citet{firstunlearning}, and the field has since grown into two main categories: \textit{exact unlearning} (retraining from scratch) and \textit{approximate unlearning} (reducing the cost of retraining while matching the from-scratch reference up to a tolerance). MU was later extended to Federated Learning as Federated Unlearning (FU)~\cite{fu_survey}, with heavy study in both server-orchestrated~\cite{pmlr-v54-mcmahan17a} and fully decentralized scenarios. The decentralized nature of FL has led many works to adopt relaxed assumptions, and obtaining rigorous certified unlearning guarantees~\cite{DBLP:conf/nips/SekhariAKS21,DBLP:conf/iclr/AllouahKGK25} is difficult; results often hold only under (strongly) convex objectives, while many server-based and decentralized FU methods~\cite{halimi, manaar, elasticsga, federaser_fu3, dec_un1,dec_un2,dec_un3} are heuristic and lack theoretical guarantees.

\citet{fedrecovery} provided Differential Privacy (DP)~\cite{DBLP:journals/fttcs/DworkR14}-based unlearning guarantees for server-orchestrated FL via perturbed improved retraining where all clients and the server collaborate, which is computationally expensive by design. On the other hand,~\citet{qiao2025pdudt} gave the first certified unlearning for fully decentralized unlearning under dynamic topologies, introducing two algorithms that involve all clients or all neighbors and require storing past gradients; this induces high storage/compute overhead and temporal test-accuracy drops during unlearning. In addition, that work does not articulate decentralized \emph{view}-specific issues in certification, lacks a deletion-capacity analysis~\cite{DBLP:conf/nips/SekhariAKS21,DBLP:conf/iclr/AllouahKGK25,DBLP:journals/corr/abs-2502-17323} that formalizes the utility–unlearning trade-off, and does not separate DP-based guarantees from certified unlearning as done in centralized settings~\cite{DBLP:conf/nips/SekhariAKS21,DBLP:conf/iclr/AllouahKGK25,DBLP:journals/corr/huang_cannone}. Altogether, certified decentralized unlearning on fixed graphs remains underexplored and needs a rigorous reformulation.

Following this line of work, and inspired by privacy amplification from decentralization~\cite{DBLP:conf/aistats/CyffersB22} and subsampling~\cite{DBLP:conf/nips/BalleBG18,DBLP:conf/focs/FeldmanMTT18}, we address limitations in certified decentralized unlearning. We ask: can unlearning be performed \emph{autonomously} without involving all clients on a fixed topology; what are the effects of decentralization (routing, mixing, trust regions) on the utility–privacy trade-off and on deletion capacity; and is Decentralized DP (DDP) still a good candidate for certified unlearning? To answer these questions, we reformulate certified unlearning in decentralized settings where each client has only a partial view of the network; we then introduce a new approximate decentralized certified unlearning algorithm on a fixed topology that combines projected noisy gradient ascent on the unlearning client only with randomized client sampling, called \methodname\ (randomized-restart decentralized unlearning), while other clients continue regular training. We incorporate the relevant amplification analyses (network mixing, projection, subsampling)~\cite{DBLP:conf/aistats/CyffersB22,DBLP:conf/nips/BalleBG18,DBLP:conf/focs/FeldmanMTT18}, show why DDP is not an ideal unlearning certifier due to noise scaling with the forget-set size \(m\), and prove that \methodname\ attains stronger guarantees with substantially less noise (not scaling with \(m\)), together with unlearning–utility and deletion-capacity analyses for both settings; we also provide convergence guarantees for strongly convex, convex, and smooth nonconvex losses.

We support our claims with experiments on \emph{two real-world datasets} and \emph{two model architectures}, showing that we drive forget accuracy down to the retraining baseline (near random guessing) while maintaining higher retained accuracy—demonstrating that \methodname\ is practical with low communication and storage overhead.

Our \textbf{\underline{contributions}} can be summarized as follows:
\begin{itemize}
  \item \textit{Formulation \& capacity (first to our knowledge).} We introduce a \emph{view-based} formulation of certified decentralized unlearning on fixed graphs and adapt deletion capacity to decentralized settings, making explicit the roles of routing probabilities, network mixing, and trust-region projections.
  \item \textit{Algorithm.} We propose \methodname: a lightweight token method that performs projected noisy ascent on the forget set \emph{only at the unlearning client}, uses randomized routing with optional per-node averaging, and requires neither storing past gradients nor involving all clients.
  \item \textit{Theory.} We derive $(\varepsilon,\delta)$ decentralized certificates on client \emph{views} via RDP with subsampling and network amplification; prove last-iterate/stationarity guarantees for strongly convex, convex, and smooth nonconvex losses; and give a two-regime deletion-capacity characterization that separates optimization/variance from an alignment-bias term—clarifying when \methodname\ outperforms DDP. This also yields a clean separation between DDP-as-certifier (group-privacy scales with the number of deletions) and approximate certified unlearning, for which \methodname’s noise does not scale with the forget-set size.
  \item \textit{Empirics \& code.} Under matched privacy budgets, \methodname\ attains near-random forget accuracy and better retained utility than DDP and fine-tuning baselines; we will release code to facilitate reproduction and extension.
\end{itemize}

\section{Related Work}
\paragraph{Private Decentralized Learning}
Motivated by privacy and scalability concerns~\cite{inverting,DBLP:journals/corr/ShokriSS16}, a large body of work studies \emph{fully decentralized} FL algorithms that avoid a central coordinator and rely instead on peer-to-peer exchanges along edges of a network graph—either under dynamic topologies~\cite{Lian2017b,Lian2018,koloskova2020unified} or fixed topologies~\cite{neglia2020}. Decentralized stochastic gradient methods (often called incremental/gossip methods) operate via gossip protocols~\cite{random_gossip,DBLP:conf/icml/ColinBSC16}. To satisfy privacy requirements, Differential Privacy (DP)—originally developed in the centralized model (DP-SGD~\cite{DBLP:conf/ccs/AbadiCGMMT016})—has been adapted to decentralized settings. In particular, \citet{DBLP:conf/aistats/CyffersB22} introduced \emph{network differential privacy} to relax the stronger local DP model~\cite{Kasiviswanathan2008,d13}, leveraging privacy amplification by subsampling~\cite{Balle_subsampling}, by shuffling~\cite{Balle2019,amp_shuffling,Balle2019b}, and by iteration/contractive maps (e.g., Euclidean projection)~\cite{DBLP:conf/focs/FeldmanMTT18}, together with Rényi DP~\cite{RDP}, to obtain \emph{amplification by decentralization} on fixed topologies. Recent work further tightens convergence and privacy via matrix factorization views~\cite{bellet2025unifiedprivacyguaranteesdecentralized}, correlated noise~\cite{DBLP:conf/icml/AllouahKFJG24}, and synthetic-data amplification~\cite{DBLP:journals/corr/abs-2506-05101}. While DP is a strong notion that implies certified unlearning (CU) through group privacy and composition~\cite{DBLP:books/sp/17/Vadhan17}, our focus here is to adopt the perspective of~\cite{DBLP:conf/aistats/CyffersB22} and show \emph{why decentralized DP (DDP) is not an ideal certifier for unlearning}, despite being theoretically valid. We make a clean separation between DDP and decentralized certified unlearning (DCU).
\vspace{-4pt}
\paragraph{Certified Unlearning}
There is a growing literature on CU in centralized settings: many results target convex objectives and linear models~\cite{DBLP:conf/nips/SekhariAKS21,DBLP:conf/icml/GuoGHM20,DBLP:conf/nips/GinartGVZ19,DBLP:journals/corr/abs-2502-17323}, while others extend to non-convex tasks~\cite{DBLP:conf/icml/ChourasiaS23,DBLP:conf/nips/ChienWCL24a,DBLP:conf/iclr/AllouahKGK25,DBLP:journals/corr/abs-2409-09778} and neural networks~\cite{DBLP:journals/corr/abs-2506-06985,DBLP:conf/icml/ZhangDWL24}. In fully decentralized FL, the only work studying CU under random topologies and gossip is~\cite{qiao2025pdudt}, which uses the Gaussian mechanism but involves all users (or all neighbors) and stores past gradients—incurring significant storage/compute overhead and a temporary drop in test accuracy during unlearning. Moreover, it imports CU definitions from the centralized setting and certifies \emph{per client}, without extending guarantees to network-level \emph{views}. In contrast, we add noise \emph{only at the unlearning client} and perform \emph{noisy projected gradient ascent} there, while others continue standard training; randomized next-hop selection and per-hop step counts act as \emph{post-processing}, amplifying privacy and yielding tighter bounds without relying on DDP. Our gradient-alignment design makes the required noise \emph{independent of the forget-set size}, improving the privacy–utility trade-off. Because the protocols differ fundamentally (dynamic topologies and gossip in~\cite{qiao2025pdudt} vs.\ a \emph{random walk} on fixed graphs here), we do not compare directly; head-to-head comparisons are misleading in this scenario due to the fundamental differences in settings.

\section{Problem Statement}
\label{sec:problem_statement}

Let $\mathcal{G}=(\mathcal{V},\mathcal{E})$ be a fully connected, undirected graph with  $\mathcal{V}=\{1,\dots,N\}$ users; an edge $(u,v)\in\mathcal{E}$ means $u$ can communicate with $v$. 
Each user $u$ holds a private dataset $D_u$ of size $n_u:=|D_u|$. We denote the global dataset on the graph by $D:=\bigcup_{u\in\mathcal{V}} D_u$ with total size $n:=|D|=\sum_{u} n_u$, which is independently drawn from a distribution~$\mathcal{P}$ over the data
space~$\mathcal{Z}$.
Users want to collaboratively learn a shared model represented by a parameter vector $\theta \in \Theta \subseteq \R^d$.
Given a loss function $\ell: \R^d \times  \mathcal{Z} \rightarrow \R_+$, the goal is to minimize the population risk:
\begin{equation}
    \mathcal{L}( \theta ) :=\E_{z\sim \mathcal{P}}[\,\ell(\theta,z)\,].
\end{equation}

However, during the \textit{training} phase, we reduce the problem to empirical risk minimization:
\begin{equation}
\label{eq:objective}
\min_{\theta\in\Theta} \mathcal{L}(\theta,\mathcal{V})
:= \frac{1}{|\mathcal{V}|} \sum_{u\in\mathcal{V}} \ell_u(\theta),
\end{equation}
where $\ell_u$ is the local objective function of user $u$, defined as
\begin{equation}
\label{eq:local_objectifs}
\ell_u(\theta):=\frac{1}{n_u}\sum_{z\in D_u}\ell(\theta;z).
\end{equation}

We denote by $\mathcal{A}$ our training algorithm. It is basically implemented as a single \emph{token} carrying the current model $\theta$ that performs a random walk on $\mathcal{G}$ for $T$ rounds. When user $u$ receives the token, it runs local SGD steps on $D_u$ and forwards the token (with updated $\theta$) to a uniformly random neighbor. Each user updates only upon token arrival (see Algorithm~\textbf{1} in Appendix~\textbf{B.1}). 

Moving to the \textit{unlearning} scenario: at some round $t \in \{1,\dots,T\}$, user $u$ receives a deletion request for a subset $D_f\subseteq D_u$ with $m:=|D_f|$.
Retraining from scratch on $D\setminus D_f$ is the natural approach, but it is often impractical; instead we apply an \textit{unlearning algorithm} $\mathcal{U}$ to transform the output of the original algorithm, $\mathcal{A}(D)$, into an \emph{unlearned} model that is distributionally close to the output of a suitable certifier that has no access to $D_f$. This is known as \textit{certified unlearning}; formally defined as follows.
\begin{definition}[$(\varepsilon,\delta)$-Certified Unlearning]
\label{def:approx_unlearning}
Let $D$ be a dataset of size $n$ drawn from a distribution $\mathcal{P}$, and let $D_f\subseteq D$ be a delete set with $|D_f|\le m$.
Let $\mathcal{A}$ be a learning algorithm that outputs $\mathcal{A}(D)\in\Theta$, and let $\mathcal{U}$ be an unlearning algorithm that, given a delete (forget) set $D_f$, a model, and data statistics $T(D)$, outputs $\mathcal{U}(D_f, \mathcal{A}(D), T(D))\in\Theta$.
We say that $(\mathcal{A},\mathcal{U})$ is \emph{$(\varepsilon,\delta)$-unlearning} if there exists a (possibly problem-dependent) \emph{certifying algorithm} $\mathcal{C}$ such that for all measurable sets $\theta \subseteq \Theta$:
\[
\begin{aligned}
& \Prob\!\big[ \mathcal{U}(D_f, \mathcal{A}(D), T(D)) \in \theta \big] \le e^{\varepsilon}\,\Prob\!\big[\mathcal{C}(D\setminus D_f) \in \theta \big] + \delta,\\[-0.25ex]
& \Prob\!\big[\mathcal{C}(D\setminus D_f) \in \theta \big] \le e^{\varepsilon}\,\Prob\!\big[ \mathcal{U}(D_f, \mathcal{A}(D), T(D)) \in \theta \big] + \delta.
\end{aligned}
\]
Thus, the output distribution after unlearning is $(\varepsilon,\delta)$-indistinguishable from that of a certifying procedure that has no access to the forget set $D_f$. 
Typical choices include $\mathcal{C}(D\setminus D_f)= \mathcal{A}(D\setminus D_f)$~\citet{DBLP:conf/nips/GinartGVZ19, DBLP:conf/icml/GuoGHM20} or 
$\mathcal{C}(D\setminus D_f)= \mathcal{U}(\varnothing, \mathcal{A}(D\setminus D_f), T(D\setminus D_f))$~\citet{DBLP:conf/nips/SekhariAKS21, DBLP:conf/iclr/AllouahKGK25}. In this work, we adopt the second choice, noting that this is just a theoretical choice and does not affect the results in any way.
\end{definition}
\vspace{-2pt}
\paragraph{Decentralized Formulation.}
We adopt decentralized differential privacy (DDP) of \citet{DBLP:conf/aistats/CyffersB22}. A decentralized algorithm $\mathcal{A}$ on a graph produces a transcript $\mathcal{A}(D)$ of all exchanged messages. No user sees the full transcript: user $u$ only observes a \emph{view}
\begin{equation}
\label{eq:view}
O_u(\mathcal{A}(D)) \;=\; \{\, (v,m,v') \in \mathcal{A}(D)\;:\; v=u \ \text{or}\ v'=u \,\}.
\end{equation}
For each $v$, let $\Theta_v := \mathrm{Range}(O_v)$ denote the observation space of $v$'s views.

\begin{definition}[Network Differential Privacy~\cite{DBLP:conf/aistats/CyffersB22}]
\label{def:network_dp}
An algorithm $\mathcal{A}$ satisfies $(\varepsilon,\delta)$-network DP if for all distinct $u,v\in\mathcal{V}$, all neighboring datasets $D \sim_u D'$ (differing only in user $u$'s data), and all $\theta \subseteq \Theta_v$,
\begin{equation}
\label{eq:network_dp}
\Prob\!\big[\,O_v(\mathcal{A}(D)) \in \theta\,\big]
\;\le\;
e^{\varepsilon}\,\Prob\!\big[\,O_v(\mathcal{A}(D')) \in \theta\,\big] + \delta.
\end{equation}
\end{definition}

We adapt certified unlearning to this view-based setting:

\begin{definition}[$(\varepsilon,\delta)$-Decentralized Certified Unlearning]
\label{def:decentralized_unlearning}
Let $\mathcal{A}$ produce $\mathcal{A}(D)$ and let $\mathcal{U}$ produce $\mathcal{U}(D_f,\mathcal{A}(D))$. We say $(\mathcal{A},\mathcal{U})$ achieves $(\varepsilon,\delta)$ \emph{decentralized certified unlearning} if there exists a certifier $\mathcal{C}$ with transcript $\mathcal{C}(D\setminus D_f)$ such that for any deletion request by user $u$ (i.e., $D_f\subseteq D_u$), any $v\neq u$, and all $\theta \subseteq \Theta_v$,
\begin{equation}
\Prob\!\big[\,O_v(\mathcal{U}(D_f,\mathcal{A}(D))) \in \theta\,\big]
\;\le\;
e^{\varepsilon}\,\Prob\!\big[\,O_v(\mathcal{C}(D\setminus D_f)) \in \theta\,\big] + \delta,
\end{equation}
and \emph{symmetrically} with $\mathcal{U}$ and $\mathcal{C}$ swapped.
\end{definition}

To quantify the trade-off between the statistical guarantee of \Cref{def:approx_unlearning} and the resulting model's utility, we adopt the \emph{Deletion Capacity} metric from \citet{DBLP:conf/nips/SekhariAKS21}. \\

\begin{definition}[Deletion capacity]
\label{def:deletion_capacity}
Let $\varepsilon,\delta\ge 0$. Let $D\sim \mathcal{P}^n$ be drawn i.i.d.\ from a distribution $\mathcal{P}$, and let $\ell(\theta,z)$ be a loss. Define the population risk $\LF( \theta )=\E_{z\sim \mathcal{P}}[\,\ell(\theta,z)\,]$ and $\LF^\ast=\min_{\theta \in \Theta }\LF(\theta)$. For a pair $(\mathcal{A}, \mathcal{U})$ that is $(\varepsilon,\delta)$-unlearning (per \Cref{def:approx_unlearning} or \Cref{def:decentralized_unlearning}), the \emph{deletion capacity}
$m^{\mathcal{A}, \mathcal{U}}_{\varepsilon,\delta}(d,N)$ is the largest integer $m$ such that
\[
\E\!\left[\;\max_{D_f\subseteq D:\,|D_f|\le m}\; \big(\LF(\mathcal{U}(D_f, \mathcal{A}(D), T(D))) - \LF^\ast\big)\;\right]\;\le\; \gamma,
\]
where the expectation is taken over $D\sim \mathcal{P}^n$ and over the internal randomness of $\mathcal{A}$ and $\mathcal{U}$ (and any randomness in $T$).
\end{definition}

\paragraph{Unlearning via Differential Privacy.} The view-based definitions let us certify unlearning by privacy: Differential Privacy (DP) implies global certified unlearning (\Cref{def:approx_unlearning}) with certifier $\mathcal{C}(D\!\setminus\!D_f)=\mathcal{A}(D\!\setminus\!D_f)$, and Network-DP (\Cref{def:network_dp}) implies decentralized certified unlearning on \emph{views} (\Cref{def:decentralized_unlearning}). Thus any network-private token algorithm (see Algorithm~\textbf{2} in Appendix~\textbf{B.2}) can serve as its own certifier by simply running the private protocol: $\mathcal{U}(D_f,\mathcal{A}(D))=\mathcal{A}(D)$ and comparing to $\mathcal{C}(D\!\setminus\!D_f)=\mathcal{A}(D\!\setminus\!D_f)$. We now quantify the deletion capacity (\Cref{def:deletion_capacity}) achievable under this decentralized-DP baseline.

\begin{theorem}[Deletion capacity of decentralized DP]
\label{thm:dcap-cb}
Let $\Theta\subset\R^{d}$ be convex with diameter $R:=\sup_{\theta,\theta'\in\Theta}\|\theta-\theta'\|_{2}$, and assume the loss $\ell$ is $L$-smooth and convex. Consider the network-private SGD (Appendix~\textbf{B.2}), run for $T$ hops with $N$ clients, and take $\mathcal{U}(D_f,\mathcal{A}(D),T(D))=\mathcal{A}(D)$. Then, for target $(\varepsilon,\delta)$ at edit distance $m$, the deletion capacity satisfies
\begin{align}
\label{eq:ddp-capacity}
m^{\mathcal{A},\mathcal{U}}_{\varepsilon,\delta}(d,N)
\;=\;
\widetilde{\Omega}\!\left(
\frac{\varepsilon}{R\,L\,(2+\log T)}\;
\sqrt{\frac{N}{d\,\log(1/\delta)\,\log N}}
\right).
\end{align}
\end{theorem}

\paragraph{Proof sketch.}
We combine utility, decentralized view calibration, and group privacy, then solve for $m$. The full proof is deferred to Appendix~\textbf{C.1}.
\vspace{-2pt}
\paragraph{Decentralization Effect on Deletion Capacity.}
For the DDP baseline~\cite{DBLP:conf/aistats/CyffersB22}, ~\Cref{thm:dcap-cb} shows that the view-based calibration yields the lower bound in \eqref{eq:ddp-capacity}. In words, capacity \emph{increases} with the number of users $N$ and \emph{decreases} with the model dimension $d$, failure probability $\delta$, and the optimization horizon via the $(2+\log T)$ factor. Operationally, $T$ tracks the total number of stochastic updates observed across the network (each user contributes about $T/N$ in expectation). By contrast, in centralized DP, classical analyses give analogous deletion-capacity guarantees where the key driver is the total dataset size $n$ rather than the number of nodes $N$ \cite{DBLP:conf/nips/SekhariAKS21,DBLP:journals/corr/huang_cannone,DBLP:conf/iclr/AllouahKGK25}. Full details are deferred to Appendix~\textbf{C.2}.

If each random-walk message aggregates $s\!\ge\!1$ independent unbiased gradients with independent Gaussian noise terms \emph{before} any observer’s first view, the effective variance becomes $G^2=L^2+d\sigma^2/s$, which improves capacity by a factor $\sqrt{s}$. The corresponding scaling is
\begin{align}
m^{\mathcal{A},\mathcal{U}}_{\varepsilon,\delta}(d,N)
\;=\;
\widetilde{\Omega}\!\left(
\frac{\varepsilon}{R\,L\,(2+\log T)}\,
\sqrt{\frac{s\,N}{d\,\log(1/\delta)\,\log N}}
\right).
\end{align}
Two special cases recover common protocols: (i) \emph{one-update-per-hop} random walk ($s{=}1$) reduces to \eqref{eq:ddp-capacity}; (ii) synchronous rounds that average over all users ($s{\approx}N$) yield an extra $\sqrt{N}$ gain relative to one-update-per-hop.

\section{Algorithm Design}
\label{sec:alg-design}

\subsection{Motivation and Problem Setup}
As discussed in \Cref{sec:problem_statement}, Decentralized Differential Privacy (DDP)~\cite{DBLP:conf/aistats/CyffersB22} can deliver decentralized certified unlearning (DCU), but it typically injects noise at \emph{every} step, which degrades utility. In our random-walk setting, the forget set $D_f\!\subseteq\!D_u$ affects the computation only when the walker is at the unlearning user $u$; updates performed at other users are independent of $D_f$ and thus constitute \emph{post-processing} with respect to $D_f$. This observation motivates our design: add Gaussian noise \emph{only at $u$}, keep other users noiseless, and exploit network mixing so that any given observer sees only a limited fraction of sensitive events. The result is a DCU mechanism with markedly smaller effective variance than network-wide DDP under the same $(\varepsilon,\delta)$ view-privacy budget.

\subsection{Mathematical Foundation: Gradient Alignment}
We work with the user-averaged empirical risk
\begin{align}
\mathcal{L}(\theta,\mathcal{V})
:= \frac{1}{N}\sum_{v\in\mathcal{V}} \ell_v(\theta),
\quad
\ell_v(\theta)
:= \frac{1}{n_v}\sum_{z\in D_v} \ell(\theta;z).
\end{align}
After deleting $D_f\!\subseteq\!D_u$ of size $m$, the retraining objective becomes
\begin{equation}
\begin{aligned}
\mathcal{L}_{\setminus f}(\theta,\mathcal{V})
&:= \frac{1}{N}\!\left(\ell_{u\setminus f}(\theta)+\!\!\sum_{v\neq u}\!\ell_v(\theta)\right),\\
\ell_{u\setminus f}(\theta)
&:= \frac{1}{n_u-m}\!\!\sum_{z\in D_u\setminus D_f}\!\!\ell(\theta;z).
\end{aligned}
\end{equation}

Let $\ell_f(\theta):=\frac{1}{m}\sum_{z\in D_f}\ell(\theta;z)$ and assume $L$-smoothness ($\|\nabla\ell(\theta;z)\|_2\le L$). A direct calculation gives
\begin{equation}
\begin{aligned}
\nabla\mathcal{L}_{\setminus f}(\theta,\mathcal{V})
&= \nabla\mathcal{L}(\theta,\mathcal{V})
- \frac{1}{N}\Big[\nabla\ell_u(\theta)-\nabla\ell_{u\setminus f}(\theta)\Big]
\\
&\quad + \Delta_{\mathrm{norm}}(\theta).
\end{aligned}
\end{equation}

where $\|\Delta_{\mathrm{norm}}(\theta)\|_2 = O\!\left(L\,\frac{m}{n_u}\right).$
Thus, moving toward the retrained optimum can be realized by standard descent on $v\!\neq\!u$ plus a \emph{corrective step at $u$} that depends only on local information.

\paragraph{Decentralized Realization via Random Walk}
We implement \methodname~as Algorithm~\ref{alg:rrdu}: the walker visits $u$ with probability $p$ and otherwise a uniformly sampled node in $\mathcal{V}\!\setminus\!\{u\}$; non-$u$ nodes may average $s\!\ge\!1$ mini-batches. Noise is injected \emph{only} at $u$.

\begin{algorithm}[t]
\caption{\methodname: Randomized-Restart Decentralized Unlearning}
\label{alg:rrdu}
\begin{algorithmic}[1]
\Require Initial $\theta_0$; unlearning user $u$; forget set $D_f\subseteq D_u$; routing prob.\ $p$; stepsizes $\{\eta_t\}$; noise scale $\sigma$; feasible set $\Theta$; trust ball $\mathbb{B}(\theta_{\mathrm{ref}},\varrho)$; horizon $T_u$; local-averaging $s\!\ge\!1$; \textbf{mode} $\in\{\textsc{exact},\textsc{lightweight}\}$
\State $\theta \leftarrow \theta_0$
\For{$t=1$ \textbf{to} $T_u$}
  \State \textbf{Route}: with prob.\ $p$ move to $u$; else to $v\sim\mathrm{Uniform}(\mathcal{V}\setminus\{u\})$
  \If{current node is $u$} \Comment{\emph{Noisy} corrective step at $u$}
     \State draw $Z_t \sim \mathcal{N}(0,\sigma^2 I_d)$
     \If{\textbf{mode} $=$ \textsc{exact}}
        \State $g_u \leftarrow -\,\nabla\ell_{u\setminus f}(\theta)$
     \Else \Comment{\textsc{lightweight}}
        \State $g_u \leftarrow \frac{m}{n_u}\,\nabla\ell_{f}(\theta)$
     \EndIf
     \State $\theta \leftarrow \Pi_{\mathbb{B}(\theta_{\mathrm{ref}},\varrho)}\!\big(\theta + \eta_t\,(g_u + Z_t)\big)$
  \Else \Comment{\emph{Noiseless} PGD on $D_v$}
     \State draw $B_v^{(1)},\dots,B_v^{(s)}\subseteq D_v$ i.i.d.
     \State $g_v \leftarrow \frac{1}{s}\sum_{i=1}^{s}\nabla\ell(\theta; B_v^{(i)})$
     \State $\theta \leftarrow \Pi_{\Theta}\!\big(\theta - \eta_t\,g_v\big)$
  \EndIf
\EndFor
\State \textbf{return} $\theta$
\end{algorithmic}
\end{algorithm}

Choosing $p\!=\!1/N$ matches the natural visit rate of a uniform random walk and yields clean alignment weights; the trust-region projection $\Pi_{\mathbb{B}(\theta_{\mathrm{ref}},\varrho)}$ stabilizes the noisy ascent at $u$.
\vspace{-4pt}
\paragraph{Exact vs.\ Lightweight Alignment at \texorpdfstring{$u$}{u}}
\label{subsec:exact-vs-light}
When the walker is at $u$, we use either \emph{exact alignment} ($-\,\nabla\ell_{u\setminus f}(\theta)$) or \emph{lightweight alignment} $\big(\tfrac{m}{n_u}\nabla\ell_f(\theta)\big)$. Both fit the view-based privacy accounting; the lightweight choice trades a controlled alignment bias for lower compute/memory.

\subsection{Synthesis of Advantages}
Let $g_{\neg u}(\theta):=\frac{1}{N-1}\sum_{v\neq u}\nabla\ell_v(\theta)$. The conditional expected update satisfies
\begin{align}
\frac{\E[\Delta\theta_t \mid \theta_t]}{\eta_t}
= -(1-p)\,g_{\neg u}(\theta_t) + p\,g_u(\theta_t).
\end{align}
With $p=\tfrac{1}{N}$ and the \emph{exact-alignment} choice $g_u(\theta)=-\nabla\ell_{u\setminus f}(\theta)$, we recover the retraining direction:
\begin{align}
\frac{\E[\Delta\theta_t \mid \theta_t]}{\eta_t}
&= -\frac{1}{N}\!\left(\sum_{v\neq u}\nabla\ell_v(\theta_t)+\nabla\ell_{u\setminus f}(\theta_t)\right) \notag\\
&= -\,\nabla\mathcal{L}_{\setminus f}(\theta_t,\mathcal{V}).
\end{align}

With the \emph{lightweight} choice $g_u(\theta)=\tfrac{m}{n_u}\nabla\ell_f(\theta)$, the alignment error is controlled:
\begin{align}
\Big\|\tfrac{\E[\Delta\theta_t \mid \theta_t]}{\eta_t} + \nabla\mathcal{L}_{\setminus f}(\theta_t,\mathcal{V})\Big\|_2
= O\!\left(L\,\tfrac{m}{n_u}\right).
\end{align}

\paragraph{Effective Variance: Concentrated Noise and Local Averaging}
\label{subsec:effective-variance}
Only a fraction $p$ of hops add Gaussian noise, so the per-hop second moment obeys
\begin{align}
\overline{G}^2
:= \frac{1}{T_u}\sum_{t=1}^{T_u}\E\!\big[\|g_t\|_2^2\big]
\;\le\; L^2 \;+\; \frac{p}{s}\,d\,\sigma^2,
\end{align}
where local averaging $s$ reduces stochastic variance at non-$u$ users by $\approx 1/s$. In contrast, network-wide DDP adds $d\,\sigma^2$ on every hop.

\paragraph{Geometric mixing and stability.}
For any observer $v\!\neq\!u$, the first-observation delay of a sensitive update is $\mathrm{Geom}\!\left(q\right)$ with $q \;=\; \frac{1-p}{N-1}$,
since the walker must leave $u$ (w.p\ $1-p$) and land at $v$ (w.p\ $1/(N-1)$). This mixing underpins the view-level privacy amplification used in \Cref{sec:theory}. Moreover, trust-region projection keeps noisy steps controlled:
\begin{align}
\|\theta_{t+1}-\theta_t\|_2
\;\le\; \eta_t\,\big\|g_t + Z_t\,\mathbf{1}_{\{\text{at }u\}}\big\|_2,
\end{align}
ensuring the iterate stays within $\mathbb{B}(\theta_{\mathrm{ref}},\varrho)$ while $\Pi_{\Theta}$ maintains feasibility elsewhere.

\begin{table*}[t]

\centering
\caption{\textbf{Utility bounds across objective classes (RR-DU).} Optimization and privacy terms shown separately; combine additively. Measure indicates the quantity being bounded.}
\vspace{4pt}
\footnotesize
\begin{tabular}{lccc}
\toprule
\textbf{Objective class} 
& \textbf{Optimization term} 
& \textbf{Privacy term (using Cor.~\ref{cor:noise-calib})} 
& \textbf{Measure}\\
\midrule
Convex (bounded domain) 
& $\tilde{O}\!\Big(\tfrac{R_{\rm cert}\,L}{\sqrt{s\,T_u}}\Big)$
& $\tilde{O}\!\Big(
R_{\rm cert}\,\tfrac{L}{\varepsilon}\,p\,
\sqrt{\tfrac{d\,\ln(1/\delta)\,\ln N}{s\,N}}
\Big)$
& $\LF(\theta_{T_u})-\LF^\star$\\[6pt]
$\mu$-Strongly convex 
& $\tilde{O}\!\Big(\tfrac{L^2}{\mu\,s\,T_u}\Big)$
& $\tilde{O}\!\Big(
\tfrac{L^2}{\mu}\,\tfrac{1}{\varepsilon}\,p\,
\sqrt{\tfrac{d\,\ln(1/\delta)\,\ln N}{s\,N}}
\Big)$
& $\LF(\theta_{T_u})-\LF^\star$\\[6pt]
Smooth nonconvex 
& $\tilde{O}\!\Big(\tfrac{L^2}{\sqrt{s\,T_u}}\Big)$
& $\tilde{O}\!\Big(
\tfrac{L^2}{\varepsilon}\,p\,
\sqrt{\tfrac{d\,\ln(1/\delta)\,\ln N}{s\,N}}
\Big)$
& $\tfrac{1}{T_u}\sum_{t=1}^{T_u}\|\nabla \LF(\theta_t)\|^2$\\
\bottomrule
\end{tabular}
\label{table:utilit-bounds}
\end{table*}
\section{Theoretical Analysis}
\label{sec:theory}

\subsection{Assumptions and Setup}
\label{sec:assumptions}
\textit{(Informal)} We assume: (i) $L$-smooth losses and non-expansive projections (onto $\Theta$ and the trust ball $\mathbb{B}(\theta_{\mathrm{ref}},\varrho)$); (ii) \emph{localized} Gaussian noise is injected only when the random walk is at the unlearning user $u$ and touches $D_f$ (other updates are post-processing w.r.t.\ $D_f$); (iii) routing visits $u$ with probability $p$ and otherwise a uniformly sampled node in $\mathcal{V}\!\setminus\!\{u\}$; (iv) non-$u$ nodes may average $s\!\ge\!1$ i.i.d.\ mini-batches before forwarding. Let $T_u$ denote the number of unlearning rounds (hops) and $R_{\rm cert} := \mathrm{diam}\!\big(\Theta \cap \mathbb{B}(\theta_{\mathrm{ref}},\varrho)\big).$

We analyze convex, strongly convex, and smooth nonconvex objectives; \emph{formal} statements and technical conditions appear in Appendix~\textbf{C.3}.

\subsection{Privacy on Views (Network-DP) and Noise Calibration}
\begin{theorem}[$(\varepsilon,\delta)$-DCU via view-based amplification]
\label{thm:unlearning-guarantee}
Fix $p\!\in\!(0,1)$ and horizon $T_u$. In \methodname, only visits to $u$ are sensitive. Let the single-visit Gaussian mechanism at $u$ have RDP level $\varepsilon_0$ with failure $\delta_0$. Then, for any observer $v\!\neq\!u$, the composed view-privacy parameters satisfy
\begin{equation}
\begin{aligned}
\varepsilon \;&=\;
O\!\Big(
\varepsilon_0\,\sqrt{\,p\,T_u\,\tfrac{\ln N}{N}\,\ln\tfrac{1}{\delta'}\,}
\Big),
\\
\delta \;&=\;
O\!\Big(
p\,T_u\,\delta_0\,\tfrac{\ln N}{N}
\Big)+\delta'.
\end{aligned}
\end{equation}
\end{theorem}

\begin{table*}[t]
\centering
\caption{\textbf{DDP vs.\ \methodname\ (convex case, corrected).} DDP (group privacy) requires noise that scales linearly with \(m\); \methodname\ avoids this through gradient alignment and localized noise at the unlearning user. Here \(R:=\mathrm{diam}(\Theta)\) and \(R_{\rm cert}:=\mathrm{diam}\!\big(\Theta\cap\mathbb{B}(\theta_{\mathrm{ref}},\varrho)\big)\). Hidden constants and mild \(\log T\) factors are absorbed in \(\tilde{O}(\cdot)\)/\(\tilde{\Omega}(\cdot)\).}
\vspace{4pt}
\label{tab:ddp-vs-method}
\footnotesize
\begin{tabular}{lcccc}
\toprule
\textbf{Method} 
& \textbf{Noise scale} 
& \textbf{Optimization term} 
& \textbf{Privacy term} 
& \textbf{Deletion capacity (scaling)}\\
\midrule
\textbf{DDP (group privacy)} 
& \(\Theta\!\Big(\frac{mL}{\varepsilon}\,\sqrt{\frac{T\,\ln\!\big(\frac{1}{\delta}\big)\,\ln N}{N}}\Big)\)
& \(\tilde{O}\!\Big(\frac{RL}{\sqrt{T}}\Big)\)
& \(\tilde{O}\!\Big(R\,\frac{mL}{\varepsilon}\,\sqrt{\frac{d\,\ln\!\big(\frac{1}{\delta}\big)\,\ln N}{N}}\Big)\)
& \(\tilde{\Omega}\!\Big(\frac{\varepsilon}{RL}\,\sqrt{\frac{N}{d\,\ln\!\big(\frac{1}{\delta}\big)\,\ln N}}\Big)\)\\[6pt]
\textbf{\methodname\ (ours)} 
& \(\Theta\!\Big(\frac{L}{\varepsilon}\,\sqrt{\frac{p T_u\,\ln\!\big(\frac{1}{\delta}\big)\,\ln N}{N}}\Big)\)
& \(\tilde{O}\!\Big(\frac{R_{\rm cert}L}{\sqrt{s T_u}}\Big)\)
& \(\tilde{O}\!\Big(R_{\rm cert}\,\frac{L}{\varepsilon}\,p\,\sqrt{\frac{d\,\ln\!\Big(\frac{1}{\delta}\Big)\,\ln N}{s N}}\Big)\)
& \(\gamma>A:\ \Omega\!\Big(\frac{(\gamma-A)\,n_u}{L}\Big);\ \gamma\le A:\ 0\)\\
\bottomrule
\end{tabular}
\end{table*}

\noindent\emph{Proof sketch.}
Count sensitive visits $M_u\!\sim\!\mathrm{Binomial}(T_u,p)$; use geometric first-observation and weak convexity of $D_\alpha$ to obtain the $\sqrt{\ln N/N}$ amplification; compose in RDP and convert to $(\varepsilon,\delta)$.

\begin{corollary}[Noise calibration (scaling)]
\label{cor:noise-calib}
With $\delta'=\delta/2$ and $\delta_0=\Theta\!\big(\tfrac{\delta\,N}{p\,T_u\,\ln N}\big)$, a Gaussian scale achieving $(\varepsilon,\delta)$ on views obeys
\begin{equation}
\sigma \;=\; 
\Theta\!\left(
\frac{L}{\varepsilon}\;
\sqrt{\frac{\,p\,T_u\,\ln(1/\delta)\,\ln N\,}{\,N\,}}
\right).
\end{equation}
\end{corollary}

$D_\alpha$ to obtain the $\sqrt{\ln N/N}$ amplification from Network DP; compose and convert to $(\varepsilon,\delta)$. Full proofs are deferred to~\textbf{C.3}.
\subsection{Utility: Last-Iterate Excess Risk (with averaging $s$)}
Let $\LF(\theta)$ denote the retraining objective on $D\setminus D_f$. On bounded domains, projected (stochastic) first-order methods decompose last-iterate guarantees into an \emph{optimization term} plus a \emph{variance term}. In \methodname, only a fraction $p$ of hops (those at the unlearning user $u$) inject Gaussian noise, while non-$u$ hops may average $s\!\ge\!1$ mini-batches; the per-hop second moment thus satisfies
\begin{align}
G^2 \;\le\; L^2 \;+\; \frac{p}{s}\,d\,\sigma^2.
\end{align}
Plugging the calibrated noise from Cor.~\ref{cor:noise-calib} yields the bounds summarized in Table~\ref{table:utilit-bounds} (all $\tilde{O}$ hide logarithms in $T_u$). The optimization and privacy terms add. Full proofs are provided on Appendix~\textbf{C.4}.

\subsection{Alignment Bias and Deletion Capacity}
Because RR-DU performs corrective ascent on $D_f$ only at $u$, the retraining-vs-unlearning gap includes an \emph{alignment bias}
\begin{align}
\mathrm{bias}(m)\;=\;\Theta\!\Big(L\,\frac{m}{n_u}\Big),
\end{align}
arising from the $n_u\!\to\!n_u{-}m$ renormalization and imperfect first-order alignment. Putting everything together, for a target tolerance $\gamma>0$, we require
\begin{align}
&\underbrace{\tilde{O}\!\Big(\tfrac{R_{\rm cert}L}{\sqrt{s\,T_u}}\Big)
\;+\;
\tilde{O}\!\Big(
R_{\rm cert}\,\tfrac{L}{\varepsilon}\,p\,
\sqrt{\tfrac{d\,\ln(1/\delta)\,\ln N}{s\,N}}
\Big)}_{\displaystyle A\ \text{(non-bias term, independent of $m$)}}\nonumber\\
&\quad+\;
\underbrace{C\,L\,\tfrac{m}{n_u}}_{\text{alignment bias}}
\;\le\;\gamma.
\end{align}
Hence the \emph{deletion capacity} is fundamentally two-regime:
\begin{align}
m^\star \;=\;
\begin{cases}
\displaystyle \Omega\!\Big(\tfrac{(\gamma-A)\,n_u}{L}\Big), & \text{if }\gamma>A,\\[6pt]
0, & \text{otherwise}.
\end{cases}
\end{align}

\paragraph{Regime transition.}
The variance-limited regime occurs when $A \ge \gamma$, where capacity is determined by solving:
\begin{align}
\tilde{O}\!\Big(\tfrac{R_{\rm cert}L}{\sqrt{sT_u}}\Big)
+ \tilde{O}\!\Big(R_{\rm cert}\tfrac{L}{\varepsilon}p
\sqrt{\tfrac{d\ln(1/\delta)\ln N}{sN}}\Big) \;\le\; \gamma.
\end{align}
The bias-limited regime occurs when $A < \gamma$, yielding
\begin{align}
m \;=\; \Omega\!\Big(\tfrac{(\gamma - A)\,n_u}{L}\Big).
\end{align}

\textbf{Interpretation.}
Once the non-bias term \(A\) is pushed below \(\gamma\) (by increasing \(T_u\), increasing \(s\), choosing moderate \(p\), and benefiting from larger \(N\) via the \(\sqrt{\ln N/N}\) amplification), the capacity becomes \(\Theta(\gamma n_u/L)\): it is linear in the local data size \(n_u\) and no longer improves with \(N\). Any apparent \(N\)-gain seen in earlier formulas comes from the variance-limited regime where \(A\) dominates.
\vspace{-3pt}
\subsection{Comparison with DDP Baseline }
Decentralized Differential Privacy (DDP) injects noise at every hop and relies on group privacy when certifying unlearning at edit distance \(m\), which forces the DDP noise scale to grow linearly with \(m\). In contrast, \methodname\ concentrates noise at the unlearning user and uses gradient alignment so that the noise scale does not depend on \(m\); the only \(m\)-dependence enters through the alignment bias term \(O(L\,m/n_u)\). Table~\ref{tab:ddp-vs-method} summarizes the resulting scalings (convex case), separating optimization and privacy contributions to the excess-risk bound. The key takeaway is that, once the non-bias term \(A\) is below \(\gamma\), \methodname\ admits deletion capacity linear in \(n_u\), while DDP is fundamentally limited by group privacy.

\vspace{-3pt}

\paragraph{Key insight.}
DDP’s deletion capacity is constrained by group privacy—its calibrated noise grows with \(m\)—so even under favorable mixing it cannot exploit large \(n_u\). \methodname\ removes the \(m\)-dependence from the noise scale and pays for \(m\) only through a controllable alignment bias. Consequently, once \(A\) is driven below \(\gamma\), the achievable capacity scales as \(\Theta(n_u)\).
\vspace{-3pt}

\paragraph{When \methodname\ dominates.}
\methodname\ outperforms DDP precisely in the regime where the variance-driven term \(A\) is already below the target tolerance \(\gamma\). Pushing \(A\) down can be accomplished by longer horizons \(T_u\), modest local averaging \(s\), and the natural network amplification \(\sqrt{\ln N/N}\) (with \(p\approx 1/N\)), after which capacity becomes linear in \(n_u\) and independent of \(N\). For more details on the deletion capacity of \methodname~and DDP, see Appendix~\textbf{C.5}.
\vspace{-3pt}

\paragraph{Practical notes.}
Increasing \(s\) reduces variance but does not alter privacy since noise is injected only at the unlearning user. Choosing \(p\simeq 1/N\) aligns the random-walk visit rate with uniform mixing and keeps the corrective step well-weighted. For strongly convex and smooth nonconvex objectives, replace the optimization column in Table~\ref{tab:ddp-vs-method} by \(\tilde{O}(L^2/(\mu s T_u))\) and \(\tilde{O}(L^2/\sqrt{sT_u})\), respectively; the privacy and bias scalings remain unchanged.

\begin{figure*}[t!]
  \centering

  \begin{subfigure}[t]{0.49\textwidth}
    \centering
    \includegraphics[width=1.02\linewidth]{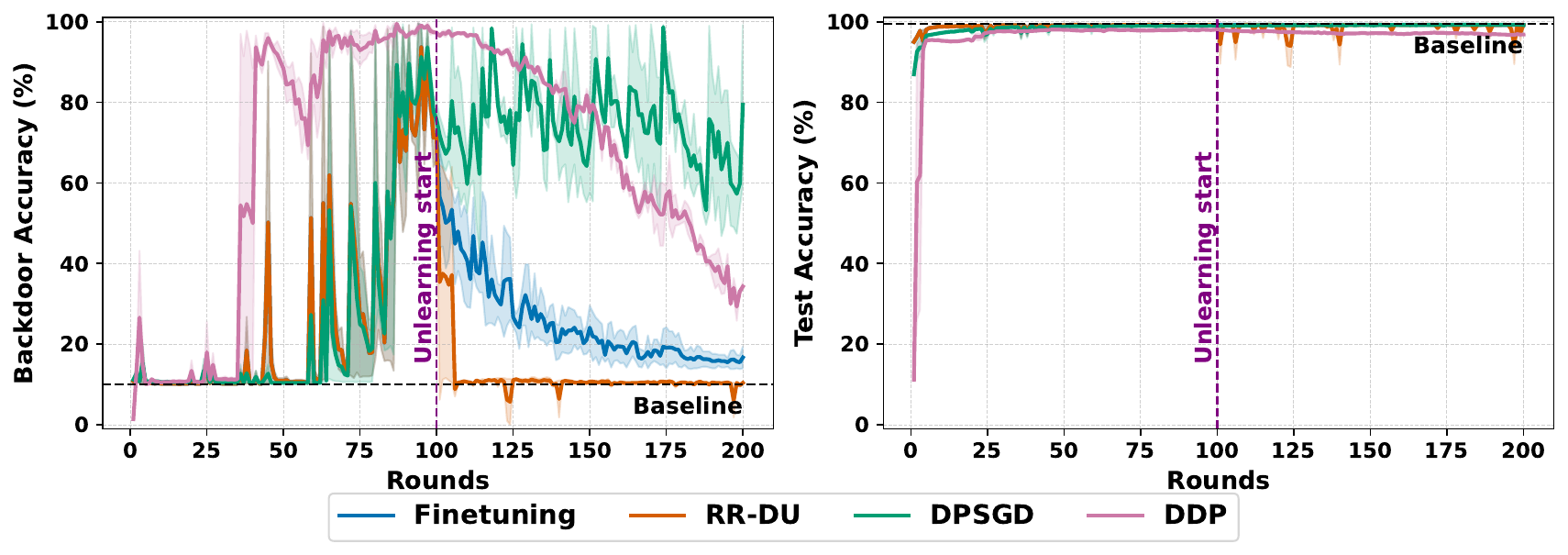}
    \caption{MNIST: backdoor and clean accuracy.}
    \label{fig:mnist_combined}
  \end{subfigure}
  \hfill
  \begin{subfigure}[t]{0.49\textwidth}
    \centering
    \includegraphics[width=1.02\linewidth]{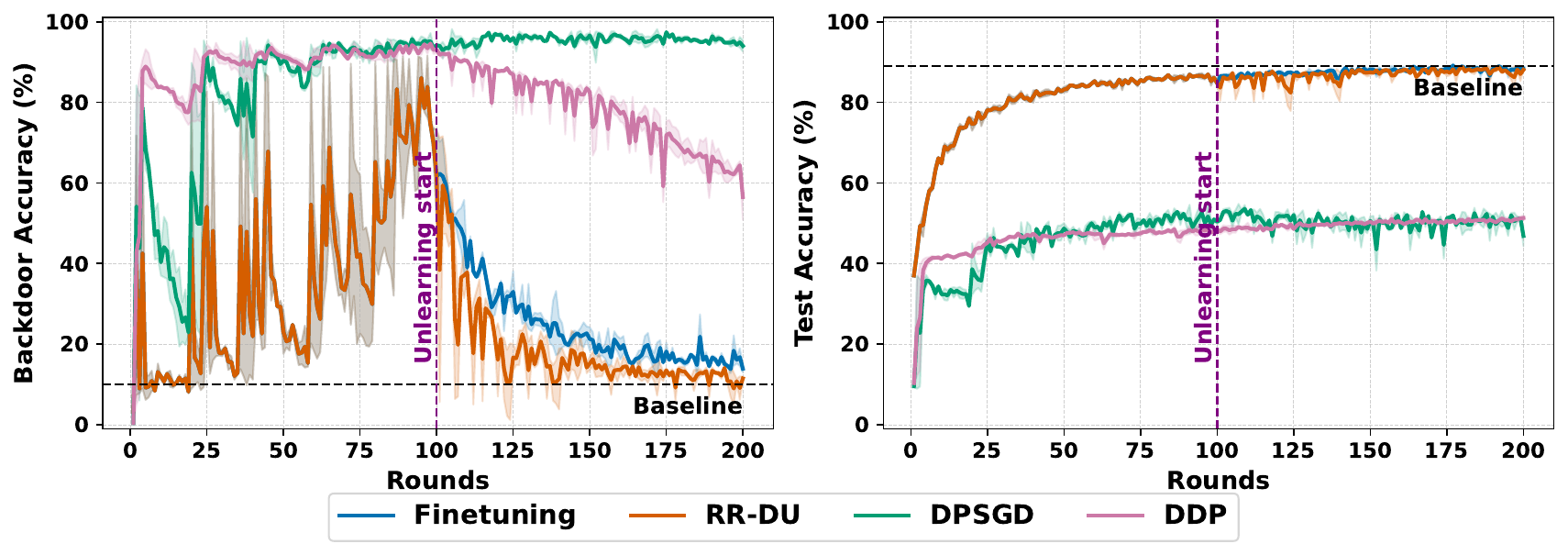}
    \caption{CIFAR-10: backdoor and clean accuracy.}
    \label{fig:cifar_combined}
  \end{subfigure}

  \caption{\textbf{Backdoor unlearning results on MNIST and CIFAR-10.}
  RR-DU vs.\ finetuning, DPSGD, and DDP. Vertical dashed lines mark unlearning
  start; horizontal dashed lines denote scratch baselines.}
  \label{fig:mnist_cifar_unlearning}
\end{figure*}

\section{Experiments}\vspace{-0.2cm}
\label{sec:experiments}
\subsection{Experimental Setup}\vspace{-0.2cm}
\label{subsec: exper_setup}
\paragraph{Dataset and Models.}
We evaluate on two standard image classification benchmarks: CIFAR-10~\citep{cifar10_100} with ResNet-18~\cite{resnet}, and MNIST~\citep{lecun1998mnist} with FLNet~\cite{DBLP:conf/aistats/McMahanMRHA17}. (See Appendix~\textbf{D} for dataset/model details.)
\vspace{-0.2cm}
\paragraph{Unlearning Scenario.}
Following prior work~\cite{halimi,manaar}, we use a backdoor (BadNets) setup~\cite{DBLP:journals/corr/abs-1708-06733} to assess unlearning effectiveness. We inject $m{=}1000$ poisoned samples into a \emph{single} target client and train for $T{=}100$ token hops, then run unlearning for $T_u{=}100$ hops. We track backdoor accuracy and clean test accuracy; the desired trade-off is to drive backdoor accuracy to $\approx 10\%$ (random-guessing baseline) while maintaining high clean accuracy. To mirror real-world conditions, poisoning is performed in a decentralized manner, which can induce fluctuations of the backdoor metric during the initial training phase; our focus is on its reduction from the unlearning round onward.
\vspace{-0.2cm}
\paragraph{Baselines.}
We compare to: (i) Decentralized DP (DDP)~\cite{DBLP:conf/aistats/CyffersB22} under the \emph{same} $(\varepsilon,\delta)$, with domain diameter $R{=}10.0$ and gradient bound $L{=}1.0$; (ii) DP-SGD~\cite{DBLP:conf/ccs/AbadiCGMMT016} with clipping $C{=}5.0$; and (iii) fine-tuning after removing the poisoned data.
\vspace{-0.2cm}
\paragraph{Network and Hyperparameters.}
We use a complete graph with $N{=}10$ clients and i.i.d.\ data split equally across clients. Unless stated otherwise, we fix $\varepsilon{=}1$ and $\delta{=}10^{-5}$. Optimization uses Adam with step size $\eta{=}0.005$, momentum $\lambda{=}0.9$, and local averaging factor $s{=}4$. The trust-region radius $\varrho$ and effective gradient bound $L$ are selected via grid search per dataset/model: $(\varrho{=}10.82,\ L{=}0.5)$ for MNIST/FLNet, and $(\varrho{=}56.30,\ L{=}0.2)$ for CIFAR-10/ResNet-18.


\subsection{Unlearning-Utility Trade-Off Evaluation}\vspace{-0.2cm}
\label{subsec: unlearning-utility}

Figure~\ref{fig:mnist_cifar_unlearning} presents the evolution of backdoor
accuracy  and clean test accuracy on MNIST
(\textbf{a}) and CIFAR-10 (\textbf{b}). The vertical purple dashed line marks
the start of unlearning, and the black dashed horizontal line denotes the
retraining-from-scratch baseline (\(\approx\!10\%\) backdoor accuracy for both
datasets, \(\approx\!99.5\%\) clean accuracy on MNIST, and
\(\approx\!89\%\) on CIFAR-10).
\vspace{-0.2cm}
\paragraph{(a) MNIST.}
Before unlearning begins (round \(\approx 100\)), all methods reach a high
backdoor success rate (\(>90\%\)), confirming that the attack is effective.
After unlearning starts, \textbf{RR-DU} rapidly drives the backdoor accuracy
down to the baseline level (\(\approx\!10\%\)), while
\textbf{finetuning} stabilizes slightly above it (\(\approx\!18\%\)).
\textbf{DPSGD} and \textbf{DDP} fail to forget the backdoor, plateauing around
\(\approx\!60\%\) and \(\approx\!35\%\), respectively.
On the clean set, RR-DU stays tightly concentrated around
\(\approx\!99.1\%-99.2\%\),
matching finetuning and outperforming DDP, which drifts downward to
\(\approx\!96.7\%\).
Overall, RR-DU achieves the closest match to the scratch baseline in both
utility and backdoor removal.
\vspace{-0.2cm}
\paragraph{(b) CIFAR-10.}
A similar pattern emerges on CIFAR-10.
Before unlearning, all methods fully learn the backdoor
(\(90\%-100\%\) ASR).
After unlearning starts, \textbf{RR-DU} again suppresses the backdoor
aggressively, converging near the baseline (\(\approx\!10\%\)).
Finetuning remains above RR-DU (\(\approx\!25\%-30\%\)),
while \textbf{DPSGD} and \textbf{DDP} retain even more backdoor signal.
On clean accuracy, RR-DU climbs steadily toward the scratch baseline
(\(\approx\!88\%-89\%\)), outperforming DPSGD and DDP, both of which saturate
around \(50\%-55\%\) and never reach baseline performance.
Finetuning recovers well but still exhibits worse backdoor removal than RR-DU. 
\vspace{-0.2cm}
\paragraph{Summary.}
Across both datasets, \methodname~consistently achieves the best trade-off: it
removes the backdoor almost as effectively as retraining from scratch while
maintaining near-optimal clean accuracy. In contrast, DP-based certifiers
(DPSGD, DDP) retain significant backdoor signal and often sacrifice clean utility.
Finetuning maintains clean accuracy but does not remove the backdoor nearly as
well as \methodname. (See Appendix~\textbf{E}  for Additional experiments).


\subsection{Effect of the Sampling Probability \texorpdfstring{$p$}{p}}
\label{subsec:effect-sampling-p}

Table~\ref{tab:p_effect} presents the effect of the sampling probability \(p\) on clean and backdoor accuracy for MNIST and CIFAR-10.

\begin{table}[t]
\centering
\caption{\textbf{Effect of sampling probability \(p\) in \methodname{} (no projection, no noise).}
\(p{=}0\): finetuning (no targeted unlearning); \(p{=}1\): \emph{continuous unlearning} (always selecting the unlearning client). Mean \(\pm\) std over 3 seeds.}
\vspace{2pt}
\scriptsize
\setlength{\tabcolsep}{5pt}
\begin{tabular}{lccc}
\toprule
\textbf{Dataset} & \textbf{\(p\)} & \textbf{Test Acc. (\%)} & \textbf{Backdoor Acc. (\%)} \\
\midrule
\multirow{3}{*}{MNIST}
& \(0\) \ (finetuning)              & \(99.18 \pm 0.04\) & \(10.33 \pm 0.09\) \\
& \(0.1\) \ (\methodname)           & \(99.18 \pm 0.04\) & \(10.33 \pm 0.09\) \\
& \(1.0\) \ (continuous unlearning) & \(10.10 \pm 0.50\) & \(0.00 \pm 0.00\)  \\
\midrule
\multirow{3}{*}{CIFAR-10}
& \(0\) \ (finetuning)              & \(88.61 \pm 0.56\) & \(10.93 \pm 0.25\) \\
& \(0.1\) \ (\methodname)           & \(88.30 \pm 0.35\) & \(11.20 \pm 0.59\) \\
& \(1.0\) \ (continuous unlearning) & \(17.74 \pm 1.77\) & \(0.00 \pm 0.00\)  \\
\bottomrule
\end{tabular}
\label{tab:p_effect}
\end{table}

\noindent\textbf{Summary.}
Across MNIST and CIFAR-10, finetuning (\(p{=}0\)) and \methodname~ with \(p{=}0.1\) reach similar endpoints (clean \(\approx\) baseline; backdoor \(\approx 10\%\)), but \methodname~ forgets \emph{faster} in rounds thanks to targeted corrective updates. At \(p{=}1\) (\emph{continuous unlearning}), the backdoor vanishes (\(0\%\)) while clean accuracy collapses, indicating over-unlearning. A small nonzero \(p\) (e.g., \(0.1\)) is the sweet spot: near-scratch forgetting with preserved utility.

\section{Conclusion and Future Work}
\noindent We formalized decentralized certified unlearning (DCU) via client \emph{views} and proposed \methodname, a random-walk mechanism that adds noise only at the unlearning client while treating other updates as post-processing. Using network differential privacy, we certify $(\varepsilon,\delta)$ on views, prove last-iterate/stationarity guarantees for convex, strongly convex, and smooth nonconvex objectives, and identify a two-regime deletion capacity separating variance from alignment bias. Compared with DDP under the same privacy budget, \methodname~avoids group-privacy scaling in the forget-set size, reduces effective variance, and—on real-world datasets—achieves near scratch-level backdoor removal with strong retained accuracy. Limitations include fixed graphs, a single walker/unlearning client, honest-but-curious observers, Gaussian/independent noise, and largely i.i.d.\ data. Future work will extend our analysis to gossip and dynamic topologies, multiple concurrent deletions, non-i.i.d.\  regimes, and communication costs.



{
    \small
    \bibliographystyle{ieeenat_fullname}
    \bibliography{main}
}
\clearpage          
\onecolumn          
\appendix

\section{Additional Notation, Definitions, and Rényi-DP Tools}
\label{appendix:notation}

This section summarizes the notation used in the main paper, restates the
unlearning and Network-DP definitions, and collects the Rényi-DP tools used
in the analysis of the DDP (NetDP) baseline and RR-DU.

\begin{table}[h]
  \centering
  \begin{tabular}{ll}
    \toprule
    Symbol & Description \\
    \midrule
    $\mathcal{G} = (\mathcal{V},\mathcal{E})$
      & Communication graph (fully connected, undirected). \\
    $\mathcal{V} = \{1,\dots,N\}$
      & Set of users (clients). \\
    $\mathcal{E}$ & Edge set of the graph. \\
    $N = |\mathcal{V}|$ & Number of users. \\
    $u,v \in \mathcal{V}$ & User indices. \\
    $(u,v) \in \mathcal{E}$ & Edge between users $u$ and $v$. \\
    $D_u \subseteq \mathcal{Z}$ & Local dataset of user $u$. \\
    $n_u = |D_u|$ & Local dataset size at user $u$. \\
    $D = \bigcup_{u \in \mathcal{V}} D_u$ & Global dataset across all users. \\
    $n = |D| = \sum_{u \in \mathcal{V}} n_u$ & Global dataset size. \\
    $D_f \subseteq D_u$ & Forget (delete) set at user $u$. \\
    $m = |D_f|$ & Size of the forget set. \\
    $D_{u \setminus f} = D_u \setminus D_f$
      & Remaining data at user $u$ after deletion. \\
    $n_{u \setminus f} = |D_{u \setminus f}|$
      & Size of $D_{u \setminus f}$. \\
    $\mathcal{Z}$ & Data space. \\
    $\mathcal{P}$ & Data distribution over $\mathcal{Z}$. \\
    $D \sim \mathcal{P}^n$ & Dataset drawn i.i.d.\ from $\mathcal{P}$. \\
    \bottomrule
  \end{tabular}
  \caption{Graph, users, and datasets.}
  \label{tab:notation-graph-users-and-datasets}
\end{table}

\begin{table}[h]
  \centering
  \begin{tabular}{ll}
    \toprule
    Symbol & Description \\
    \midrule
    $\theta \in \Theta \subseteq \mathbb{R}^d$
      & Model parameter vector. \\
    $d$ & Model dimension (number of parameters). \\
    $\Theta$ & Feasible parameter domain. \\
    $\theta_0$ & Initial model before training / unlearning. \\
    $\theta_t$ & Model after $t$ updates. \\
    $\theta_T$ & Model after $T$ updates (final iterate). \\
    $\theta^\ast$ & Population risk minimizer. \\
    $\theta_{\mathrm{ref}}$ & Trust-region center. \\
    $\varrho$ & Trust-region radius around $\theta_{\mathrm{ref}}$. \\
    $B(\theta_{\mathrm{ref}},\varrho)$
      & Ball $\{\theta:\|\theta-\theta_{\mathrm{ref}}\|_2 \le \varrho\}$. \\
    $\Theta_{\mathrm{cert}}$ &
      Certification domain $\Theta \cap B(\theta_{\mathrm{ref}},\varrho)$. \\
    $\ell(\theta;z)$ & Per-example loss at data point $z$. \\
    $\ell_u(\theta)$ & Local objective at user $u$ on $D_u$. \\
    $\ell_{u \setminus f}(\theta)$
      & Local objective at user $u$ on $D_{u \setminus f}$. \\
    $\ell_f(\theta)$ & Average loss over the forget set $D_f$. \\
    $\LF(\theta)$ & Population risk $\E_{z\sim\mathcal{P}}[\ell(\theta;z)]$. \\
    $\LF^\ast$ & Optimal population risk value. \\
    $\mathcal{L}(\theta,\mathcal{V})$
      & Empirical risk on the global dataset $D$. \\
    \bottomrule
  \end{tabular}
  \caption{Model parameters, losses, and risks.}
  \label{tab:notation-model-losses-and-risks}
\end{table}

\begin{table}[h]
  \centering
  \begin{tabular}{ll}
    \toprule
    Symbol & Description \\
    \midrule
    $\mathcal{A}$ & Training algorithm (e.g., token Net-SGD or DDP NetDP). \\
    $\mathcal{U}$ & Unlearning algorithm (e.g., RR-DU). \\
    $\mathcal{C}$ & Certifying algorithm in unlearning definitions. \\
    $T(D)$ & Auxiliary statistics of $D$ used by $\mathcal{U}$ or $\mathcal{C}$. \\
    $\mathcal{A}(D)$ & Transcript or output of $\mathcal{A}$ on $D$. \\
    $O_u(\mathcal{A}(D))$ & View (partial transcript) observed by user $u$. \\
    $\Theta_u = \mathrm{Range}(O_u)$ & Observation space of user $u$'s views. \\
    $D \sim_u D'$ & Datasets differing only in user $u$'s data. \\
    $\varepsilon$ & Privacy / unlearning parameter (multiplicative). \\
    $\delta$ & Privacy / unlearning parameter (additive slack). \\
    $\gamma$ & Excess-risk tolerance in deletion capacity. \\
    $m^{\mathcal{A},\mathcal{U}}_{\varepsilon,\delta}(d,N)$ &
      Deletion capacity of $(\mathcal{A},\mathcal{U})$ at $(\varepsilon,\delta)$. \\
    \bottomrule
  \end{tabular}
  \caption{Algorithms, views, and privacy-related notation.}
  \label{tab:notation-algorithms-views-privacy}
\end{table}

\begin{table}[h]
  \centering
  \begin{tabular}{ll}
    \toprule
    Symbol & Description \\
    \midrule
    $R$ &
      Diameter of $\Theta$:
      $R := \sup_{\theta,\theta' \in \Theta}\|\theta-\theta'\|_2$. \\
    $R_{\mathrm{cert}}$ &
      Diameter of $\Theta_{\mathrm{cert}} :=
      \Theta \cap B(\theta_{\mathrm{ref}},\varrho)$, i.e., \\
      & $R_{\mathrm{cert}}
         := \sup_{\theta,\theta' \in \Theta_{\mathrm{cert}}}\|\theta-\theta'\|_2
         \le 2\varrho$. \\
    $L$ & Lipschitz / smoothness constant of the loss (and clipping threshold). \\
    $\mu$ & Strong-convexity constant (when assumed). \\
    $G^2$ & Upper bound on $\E[\|g_t\|_2^2]$ (gradient second moment). \\
    $\sigma^2$ & Gaussian noise variance per coordinate. \\
    $\sigma$ & Gaussian noise standard deviation. \\
    $\alpha$ & Rényi-DP order in privacy analysis. \\
    $\kappa$ & Bound on inverse Hessian / condition number (when used). \\
    $\eta$ & Constant stepsize in SGD / token updates. \\
    $\eta_t$ & Stepsize at iteration / hop $t$. \\
    $T$ & Number of training token hops / rounds. \\
    $T_u$ & Number of unlearning token hops / rounds. \\
    $t$ & Iteration / hop index. \\
    $s$ & Local averaging factor (minibatches per token visit). \\
    $p$ & Routing probability toward the unlearning user. \\
    $g_t$ & Stochastic gradient estimate at step $t$. \\
    $Z_t$ & Gaussian noise vector added at step $t$. \\
    $B_u$ & Minibatch sampled from $D_u$. \\
    $B_f$ & Minibatch sampled from the forget set $D_f$. \\
    \bottomrule
  \end{tabular}
  \caption{Optimization and analysis constants.}
  \label{tab:notation-optimization-constants}
\end{table}

\begin{table}[h]
  \centering
  \begin{tabular}{ll}
    \toprule
    Symbol / Name & Description \\
    \midrule
    RR-DU & Randomized-Restart Decentralized Unlearning algorithm. \\
    DDP NetDP & Decentralized-DP baseline (network-private SGD). \\
    DP-SGD & Differentially-private SGD baseline. \\
    Fine-tuning & Retraining baseline from the pre-unlearning model (no noise). \\
    FLNet & Lightweight convolutional network used on MNIST. \\
    ResNet-18 & Residual network used on CIFAR-10. \\
    clean acc. & Test accuracy on clean (unpoisoned) examples. \\
    ASR (backdoor acc.)
      & Test accuracy on backdoor-triggered examples. \\
    $y^{\mathrm{bd}}$ & Target label for backdoor (poisoned) samples. \\
    \bottomrule
  \end{tabular}
  \caption{Baselines, models, and evaluation metrics.}
  \label{tab:notation-baselines-models-and-metrics}
\end{table}

\subsection*{Unlearning and Network-DP Definitions}

\begin{definition}[$(\varepsilon,\delta)$-Certified Unlearning (global)]
\label{def:approx_unlearning-app}
Let $D$ be a dataset of size $n$ drawn from a distribution $\mathcal{P}$, and
let $D_f \subseteq D$ be a delete set with $|D_f| \le m$. Let $\mathcal{A}$ be
a learning algorithm that outputs $\mathcal{A}(D)\in\Theta$, and let
$\mathcal{U}$ be an unlearning algorithm that, given a delete set $D_f$, a
model, and data statistics $T(D)$, outputs
$\mathcal{U}(D_f, \mathcal{A}(D), T(D))\in\Theta$.  
We say that $(\mathcal{A},\mathcal{U})$ is \emph{$(\varepsilon,\delta)$-unlearning}
if there exists a (possibly problem-dependent) \emph{certifying algorithm}
$\mathcal{C}$ such that for all measurable sets $\theta \subseteq \Theta$:
\[
\begin{aligned}
& \Prob\!\big[ \mathcal{U}(D_f, \mathcal{A}(D), T(D)) \in \theta \big]
  \le e^{\varepsilon}\,\Prob\!\big[\mathcal{C}(D\setminus D_f) \in \theta \big]
  + \delta,\\[-0.25ex]
& \Prob\!\big[\mathcal{C}(D\setminus D_f) \in \theta \big]
  \le e^{\varepsilon}\,\Prob\!\big[
    \mathcal{U}(D_f, \mathcal{A}(D), T(D)) \in \theta
  \big] + \delta.
\end{aligned}
\]
\end{definition}

\begin{definition}[Network Differential Privacy (Network-DP)]
\label{def:network_dp-app}
An algorithm $\mathcal{A}$ satisfies $(\varepsilon,\delta)$-network DP if for
all distinct $u,v \in \mathcal{V}$, all $u$-neighboring datasets $D \sim_u D'$,
and all measurable sets $\theta \subseteq \Theta_v$,
\[
\Prob\!\big[\,O_v(\mathcal{A}(D)) \in \theta\,\big]
\;\le\;
e^{\varepsilon}\,\Prob\!\big[\,O_v(\mathcal{A}(D')) \in \theta\,\big]
+ \delta.
\]
\end{definition}

\begin{definition}[$(\varepsilon,\delta)$-Decentralized Certified Unlearning]
\label{def:decentralized_unlearning-app}
Let $\mathcal{A}$ produce $\mathcal{A}(D)$ and let $\mathcal{U}$ produce
$\mathcal{U}(D_f,\mathcal{A}(D))$. We say $(\mathcal{A},\mathcal{U})$ achieves
$(\varepsilon,\delta)$ \emph{decentralized certified unlearning} if there
exists a certifier $\mathcal{C}$ with transcript $\mathcal{C}(D\setminus D_f)$
such that for any deletion request by user $u$ (i.e., $D_f \subseteq D_u$),
any $v \neq u$, and all $\theta \subseteq \Theta_v$,
\[
\Prob\!\big[\,O_v(\mathcal{U}(D_f,\mathcal{A}(D))) \in \theta\,\big]
\;\le\;
e^{\varepsilon}\,\Prob\!\big[\,O_v(\mathcal{C}(D\setminus D_f)) \in \theta\,\big]
+ \delta,
\]
and the same inequality holds with $\mathcal{U}$ and $\mathcal{C}$ swapped.
\end{definition}

\begin{definition}[Deletion capacity]
\label{def:deletion_capacity-app}
Let $\varepsilon,\delta \ge 0$. Let $D \sim \mathcal{P}^n$ be drawn i.i.d.\
from a distribution $\mathcal{P}$, and let $\ell(\theta,z)$ be a loss.
Define the population risk $\LF(\theta) = \E_{z\sim \mathcal{P}}[\,\ell(\theta;z)\,]$
and $\LF^\ast = \min_{\theta \in \Theta} \LF(\theta)$.  
For a pair $(\mathcal{A}, \mathcal{U})$ that is $(\varepsilon,\delta)$-unlearning
(in either the global or decentralized sense above), and for a fixed tolerance
$\gamma > 0$, the \emph{deletion capacity}
$m^{\mathcal{A}, \mathcal{U}}_{\varepsilon,\delta}(d,N)$ is the largest integer
$m$ such that
\[
\E\!\left[\;\max_{D_f\subseteq D:\,|D_f|\le m}\;
  \big(\LF(\mathcal{U}(D_f, \mathcal{A}(D), T(D))) - \LF^\ast\big)\;\right]
\;\le\; \gamma,
\]
where the expectation is taken over $D \sim \mathcal{P}^n$ and over the
internal randomness of $\mathcal{A}$ and $\mathcal{U}$ (and any randomness in
$T$).
\end{definition}

\paragraph*{Unlearning via Differential Privacy (recap).}
Differential Privacy implies global certified unlearning with certifier
$\mathcal{C}(D\setminus D_f)=\mathcal{A}(D\setminus D_f)$, and Network-DP
implies decentralized certified unlearning on views with the same certifier.
These reductions underpin the deletion-capacity guarantees for the DDP NetDP
baseline.

\subsection{Rényi-DP and Network-RDP Tools}
\label{appendix:rdp-tools}

We now collect the Rényi-DP tools used in the DDP NetDP and RR-DU analysis,
following Mironov~\cite{RDP}, Feldman et al.~\cite{DBLP:conf/focs/FeldmanMTT18},
and the NetDP token-SGD analysis in~\cite{DBLP:conf/aistats/CyffersB22}.

\begin{definition}[Rényi divergence {\cite{RDP}}]
\label{def:renyi-divergence}
Let $1 < \alpha < \infty$ and let $\mu,\nu$ be probability measures such that
$\mu$ is absolutely continuous with respect to $\nu$. The Rényi divergence of
order $\alpha$ between $\mu$ and $\nu$ is
\[
D_{\alpha}(\mu\|\nu)
\;:=\;
\frac{1}{\alpha-1}
\log
\int \Big(\frac{d\mu}{d\nu}(z)\Big)^{\alpha} d\nu(z).
\]
If $U\sim\mu$ and $V\sim\nu$, we often write $D_{\alpha}(U\|V)$ for
$D_{\alpha}(\mu\|\nu)$.
\end{definition}

\begin{definition}[Rényi Differential Privacy {\cite{RDP}}]
\label{def:rdp}
A randomized mechanism $\mathcal{M}$ with domain $\mathcal{X}^n$ and range
$\mathcal{Y}$ is said to satisfy $(\alpha,\varepsilon)$-Rényi Differential
Privacy if for all neighboring datasets $D,D' \in \mathcal{X}^n$ and all
$1<\alpha<\infty$,
\[
D_{\alpha}\big(\mathcal{M}(D)\,\big\|\,\mathcal{M}(D')\big)
\;\le\;
\varepsilon.
\]
\end{definition}

\begin{definition}[Network Rényi-DP]
\label{def:network-rdp}
A decentralized algorithm $\mathcal{A}$ on a graph with views
$O_v(\mathcal{A}(D))$ satisfies $(\alpha,\varepsilon)$-\emph{network
Rényi-DP} if, for every pair of distinct users $u,v \in \mathcal{V}$ and
every $u$-neighboring datasets $D \sim_u D'$,
\[
D_{\alpha}\big(O_v(\mathcal{A}(D))\,\big\|\,O_v(\mathcal{A}(D'))\big)
\;\le\;
\varepsilon.
\]
\end{definition}

\begin{proposition}[RDP $\Rightarrow$ $(\varepsilon,\delta)$-DP conversion {\cite{RDP}}]
\label{prop:rdp-to-dp}
If a mechanism $\mathcal{M}$ satisfies $(\alpha,\varepsilon)$-RDP for some
$\alpha>1$, then for every $\delta \in (0,1)$ it satisfies
$(\varepsilon',\delta)$-DP with
\[
\varepsilon'
\;=\;
\varepsilon + \frac{\log(1/\delta)}{\alpha-1}.
\]
The same implication holds in the network setting by applying this bound to
each pair of neighboring datasets and each view.
\end{proposition}

\begin{theorem}[RDP of PNSGD {\cite{DBLP:conf/focs/FeldmanMTT18,DBLP:conf/aistats/CyffersB22}}]
\label{thm:rdp-pnsgd}
Let $W\subset\mathbb{R}^d$ be convex and let $\{f(\cdot;x)\}_{x\in\mathcal{X}}$
be a family of convex, $L$-Lipschitz and $\beta$-smooth functions on $W$.
Consider Projected Noisy Stochastic Gradient Descent (PNSGD) on dataset
$D=(x_1,\dots,x_n)$:
\[
w_{t+1}
=
\Pi_{W}\big(w_t - \eta(\nabla f(w_t;x_{t+1}) + Z)\big),
\qquad
Z\sim\mathcal{N}(0,\sigma^2 I_d),
\]
run for $n$ steps with stepsize $\eta \le 2/\beta$. Then, for any order
$\alpha>1$ and any $t\in\{1,\dots,n\}$, the mechanism satisfies
$(\alpha,\varepsilon_t)$-RDP with respect to the $t$-th input, where
\[
\varepsilon_t
\;=\;
\frac{\alpha \cdot 2L^2}{\sigma^2 (n+1-t)}.
\]
\end{theorem}

\begin{proposition}[Weak convexity of Rényi divergence {\cite[Appendix A]{DBLP:conf/aistats/CyffersB22}}]
\label{prop:weak-convexity-rdp}
Let $\mu_1,\dots,\mu_m$ and $\nu_1,\dots,\nu_m$ be probability distributions
on a common measurable space, and suppose that for some $c\in(0,1]$ and all
$i\in[m]$,
\[
D_{\alpha}(\mu_i\|\nu_i)
\;\le\;
\frac{c}{\alpha-1}.
\]
Let $\rho$ be a distribution on $[m]$ and define $\mu_{\rho}$ (resp.\ $\nu_\rho$)
as the mixture distribution obtained by first sampling $i\sim\rho$ and then a
sample from $\mu_i$ (resp.\ $\nu_i$). Then
\[
D_{\alpha}(\mu_{\rho}\|\nu_{\rho})
\;\le\;
(1+c)\cdot \mathbb{E}_{i\sim\rho}\big[D_{\alpha}(\mu_i\|\nu_i)\big].
\]
\end{proposition}

\begin{lemma}[View-level RDP for token SGD on complete graphs {\cite[Theorem~4]{DBLP:conf/aistats/CyffersB22}}]
\label{lem:view-rdp-token}
Consider token-based SGD with Gaussian noise $\mathcal{N}(0,\sigma^2 I_d)$ on a
complete graph with $N$ users, as in the DDP NetDP analysis
of~\cite{DBLP:conf/aistats/CyffersB22}. Let $T_u$ be the number of visits to
user $u$, and fix distinct users $u\neq v$. Then, for each order $\alpha>1$,
there exists an absolute constant $C>0$ such that for any pair of
$u$-neighboring datasets $D\sim_u D'$,
\[
D_{\alpha}\big(Y_v\|Y_v'\big)
\;\le\;
C \cdot
\frac{\alpha L^2\,T_u\,\ln N}{\sigma^2\,N},
\]
where $Y_v$ and $Y_v'$ denote the random views of user $v$ under $D$ and $D'$,
respectively.
\end{lemma}

\begin{remark}
Lemma~\ref{lem:view-rdp-token} is the network version of amplification by
iteration for PNSGD (Theorem~\ref{thm:rdp-pnsgd}), combined with the random
walk structure and weak convexity of Rényi divergence
(Proposition~\ref{prop:weak-convexity-rdp}) as in~\cite{DBLP:conf/aistats/CyffersB22}.
\end{remark}

\begin{proposition}[Network-RDP $\Rightarrow$ Network-DP]
\label{prop:network-rdp-to-dp}
If a decentralized algorithm $\mathcal{A}$ satisfies
$(\alpha,\varepsilon)$-network RDP in the sense of
Definition~\ref{def:network-rdp}, then for every $\delta \in (0,1)$ it
satisfies $(\varepsilon',\delta)$-network DP with
\[
\varepsilon'
\;=\;
\varepsilon + \frac{\log(1/\delta)}{\alpha-1}.
\]
\end{proposition}

\begin{proof}
Apply Proposition~\ref{prop:rdp-to-dp} to each pair of neighboring datasets
$D \sim_u D'$ and each view $O_v(\mathcal{A}(D))$ separately.
\end{proof}

\section{Algorithms and Implementation Details}
\label{appendix:algorithms}

This section provides the pseudocode for the three procedures used throughout
the main paper: (i) the token-based \emph{Network-SGD} training phase, which
produces the pre-unlearning model; (ii) the \emph{Network-Private SGD} baseline
implementing DDP NetDP; and (iii) \emph{RR-DU}, our proposed randomized-restart
decentralized unlearning algorithm. All algorithms are stated in the notation
of Section~\ref{appendix:notation}.

\subsection{Network-SGD (Token Training)}

Network-SGD is the basic decentralized training routine used to obtain the
initial model $\theta_0$ before any deletion request. A single token carries
the current parameter vector $\theta$ and performs a random walk over the user
graph. Each time the token visits a user $u$, that user computes a local
stochastic gradient step on its dataset $D_u$ and then forwards the token to a
uniformly random neighbor.

Algorithm~\ref{alg:network_sgd} makes this process explicit: it takes as input
the communication graph $\mathcal{G}$, an initial token location $u_0$, the
local datasets $\{D_u\}_{u\in\mathcal{V}}$, and runs for $T$ token hops. The
final model $\theta$ returned by Algorithm~\ref{alg:network_sgd} is used as the
starting point for both the DDP NetDP baseline and RR-DU in the experiments.

\begin{algorithm}[h]
\caption{Network-SGD (token-based decentralized training)}
\label{alg:network_sgd}
\begin{algorithmic}[1]
\Require Graph $\mathcal{G} = (\mathcal{V},\mathcal{E})$,
         initial token location $u_0 \in \mathcal{V}$,
         stepsize $\eta > 0$, horizon $T \in \mathbb{N}$,
         local datasets $\{D_u\}_{u \in \mathcal{V}}$
\State Initialize model $\theta \in \Theta \subseteq \R^d$
\State $u \gets u_0$
\For{$t = 1$ to $T$}
  \State Sample minibatch $B_u \subseteq D_u$
  \State $g_t \gets \nabla_{\theta}\,\ell(\theta; B_u)$
  \State $\theta \gets \theta - \eta\, g_t$
  \State Sample neighbor
         $v \sim \mathrm{Unif}\{w \in \mathcal{V} : (u,w) \in \mathcal{E}\}$
  \State $u \gets v$ \Comment{forward the token}
\EndFor
\State \textbf{return} $\theta$
\end{algorithmic}
\end{algorithm}

\subsection{Network-Private SGD (DDP NetDP Baseline)}

Network-Private SGD is the decentralized-DP baseline (DDP NetDP) used in our
deletion-capacity analysis. It has the same token structure as Network-SGD but
adds Gaussian noise to each update and projects back onto the feasible set
$\Theta$. The noise scale is chosen as a function of $(\varepsilon,\delta)$ and
the Lipschitz bound $L$, following the DDP calibration used in the main paper.

Algorithm~\ref{alg:network_private_sgd} summarizes this procedure. At each
token hop, a user $u$ is sampled uniformly from $\mathcal{V}$, a minibatch
$B_u$ is drawn, and a noisy gradient $g_t + Z_t$ is applied before projecting
onto $\Theta$. This is the algorithm used as the DDP NetDP baseline in the
experiments and as the certifier in the theoretical results on decentralized
DP deletion capacity.

\begin{algorithm}[h]
\caption{Network-Private SGD (DDP NetDP baseline)}
\label{alg:network_private_sgd}
\begin{algorithmic}[1]
\Require Convex set $\Theta \subset \R^d$, stepsize $\eta>0$,
         horizon $T \in \mathbb{N}$, gradient bound $L$
         (i.e., $\|\nabla_\theta \ell(\theta;z)\|_2 \le L$ for all $z$),
         target privacy $(\varepsilon,\delta)$,
         datasets $\{D_u\}_{u \in \mathcal{V}}$
\State Initialize $\theta \in \Theta$
\State $\sigma^{2} \gets
  \dfrac{8 L^{2}\,\ln\!\big(1.25/\delta\big)}{\varepsilon^{2}}$
\For{$t = 1$ to $T$}
  \State Draw $u \sim \mathrm{Unif}(\mathcal{V})$ \Comment{token visit}
  \State Sample minibatch $B_u \subseteq D_u$
  \State $g_t \gets \nabla_{\theta}\,\ell(\theta; B_u)$
  \State Draw $Z_t \sim \mathcal{N}(0,\sigma^{2} I_{d})$
  \State $\theta \gets \Pi_{\Theta}\!\big(\theta - \eta\,(g_t + Z_t)\big)$
\EndFor
\State \textbf{return} $\theta$
\end{algorithmic}
\end{algorithm}

\subsection{RR-DU: Randomized-Restart Decentralized Unlearning}

RR-DU is the main unlearning algorithm proposed in the paper. Starting from a
pretrained model $\theta_0$ and a forget set $D_f \subseteq D_u$ at some user
$u$, RR-DU runs a second token phase that interleaves (i) noisy corrective
steps at the unlearning user $u$ and (ii) standard SGD steps at the other
users. The token is routed towards $u$ with probability $p$ and to a random
non-unlearning user otherwise, implementing the randomized-restart / randomized
routing mechanism.

Algorithm~\ref{alg:rrdu} gives the full implementation used in our experiments.
Corrective steps at $u$ are constrained to a trust region
$\Theta_{\mathrm{cert}} = \Theta \cap B(\theta_{\mathrm{ref}},\varrho)$ via
projection, with Gaussian noise calibrated to the desired decentralized
certified-unlearning guarantee. The algorithm supports two modes: \textsc{Exact}
(which uses the full gradient of $\ell_{u \setminus f}$) and
\textsc{Lightweight} (which approximates the alignment term using minibatches
from $D_f$). All experimental curves for \methodname{} in the main paper are
generated with variants of Algorithm~\ref{alg:rrdu}.

\begin{algorithm}[h]
\caption{RR-DU: Randomized-Restart Decentralized Unlearning}
\label{alg:rrdu}
\begin{algorithmic}[1]
\Require Initial model $\theta_0 \in \Theta$, unlearning user $u \in \mathcal{V}$,
         forget set $D_f \subseteq D_u$ of size $m$, routing probability
         $p \in (0,1]$, stepsizes $\{\eta_t\}_{t=1}^{T_u}$, noise scale
         $\sigma > 0$, feasible set $\Theta$, trust-region center
         $\theta_{\mathrm{ref}}$ and radius $\varrho > 0$,
         unlearning horizon $T_u \in \mathbb{N}$,
         local averaging factor $s \in \mathbb{N}$,
         mode $\mathrm{mode} \in \{\textsc{Exact},\textsc{Lightweight}\}$
\State $\theta \gets \theta_0$
\State Choose initial token location $v \in \mathcal{V}$ (e.g., uniformly)
\For{$t = 1$ to $T_u$}
  \State \textbf{Routing:} with probability $p$ set $v \gets u$;
         otherwise draw $v \sim \mathrm{Unif}(\mathcal{V} \setminus \{u\})$
  \If{$v = u$} \Comment{noisy corrective step at the unlearning user}
    \State Draw $Z_t \sim \mathcal{N}(0,\sigma^2 I_d)$
    \If{$\mathrm{mode} = \textsc{Exact}$}
      \State $g_u \gets -\,\nabla_\theta \ell_{u \setminus f}(\theta)$
    \Else \Comment{\textsc{Lightweight} alignment}
      \State Form minibatch $B_f \subseteq D_f$
      \State $g_u \gets \dfrac{m}{n_u}\,\nabla_\theta \ell(\theta; B_f)$
    \EndIf
    \State $\theta \gets
       \Pi_{\Theta_{\mathrm{cert}}}
       \big(\theta + \eta_t (g_u + Z_t)\big)$
  \Else \Comment{noiseless projected SGD on a non-unlearning user}
    \State Draw $s$ i.i.d.\ minibatches
           $B_v^{(1)},\dots,B_v^{(s)} \subseteq D_v$
    \State $g_v \gets
       \dfrac{1}{s}\sum_{i=1}^{s}\nabla_\theta \ell(\theta; B_v^{(i)})$
    \State $\theta \gets \Pi_{\Theta}\big(\theta - \eta_t g_v\big)$
  \EndIf
\EndFor
\State \textbf{return} $\theta$
\end{algorithmic}
\end{algorithm}

\section{Proofs of Main Results}
\label{appendix:proofs}

This section (Section~C in the supplementary material) contains detailed proofs
of the main theoretical results. We first give the full proof of the
deletion-capacity bound for the DDP NetDP baseline (Theorem~3.5), then discuss
the effect of decentralization and local averaging, and finally collect our
optimization assumptions and RR-DU utility / capacity analysis.

\subsection{Proof of Theorem~3.5 (Deletion Capacity of DDP NetDP)}
\label{appendix:proof-theorem3.5}

We start by recalling the reduction from Network-DP
(Definition~\ref{def:network_dp-app}) to decentralized certified unlearning
(Definition~\ref{def:decentralized_unlearning-app}), together with a standard
group-privacy bound. These are used in the deletion-capacity proof for the DDP
NetDP baseline~\cite{DBLP:conf/aistats/CyffersB22}.

\paragraph{Network-DP implies DCU and group privacy.}
Let $\mathcal{A}$ be a decentralized algorithm on a graph that satisfies
$(\varepsilon_0,\delta_0)$-network DP in the sense of
Definition~\ref{def:network_dp-app}. Recall that $O_v(\mathcal{A}(D))$ denotes
the view of user $v$ and $\Theta_v := \mathrm{Range}(O_v)$.

\begin{proposition}[Network-DP implies decentralized certified unlearning]
\label{prop:dp_implies_dcu}
Let $\mathcal{A}$ satisfy $(\varepsilon_0,\delta_0)$-network DP. Define
$\mathcal{U}(D_f,\mathcal{A}(D)) := \mathcal{A}(D)$ and
$\mathcal{C}(D\setminus D_f) := \mathcal{A}(D\setminus D_f)$.
Assume $D$ and $D\setminus D_f$ differ only in the data of a single user $u$.
Then, for any $v \neq u$, the pair $(\mathcal{A},\mathcal{U})$ achieves
$(\varepsilon_0,\delta_0)$ decentralized certified unlearning on views in the
sense of Definition~\ref{def:decentralized_unlearning-app}.
\end{proposition}

\begin{proof}
Fix $u \in \mathcal{V}$ and $D_f \subseteq D_u$, and set $D' := D\setminus D_f$.
Then $D \sim_u D'$. For any $v \neq u$ and measurable
$\theta \subseteq \Theta_v$, network DP gives
\[
  \Prob\big[\,O_v(\mathcal{A}(D)) \in \theta\,\big]
  \le
  e^{\varepsilon_0}\,\Prob\big[\,O_v(\mathcal{A}(D')) \in \theta\,\big]
  + \delta_0.
\]
Using $O_v(\mathcal{U}(D_f,\mathcal{A}(D))) = O_v(\mathcal{A}(D))$ and
$O_v(\mathcal{C}(D\setminus D_f)) = O_v(\mathcal{A}(D'))$ yields the first
inequality in Definition~\ref{def:decentralized_unlearning-app}. Swapping $D$
and $D'$ gives the reverse inequality.
\end{proof}

For the dependence on the edit distance $m$ (size of the forget set), we use
the usual group-privacy bound, obtained via Rényi DP and advanced composition
(see, e.g.,~\cite{RDP}).

\begin{lemma}[Group privacy]
\label{lem:group_privacy}
Suppose a mechanism (e.g., DDP NetDP) satisfies
$(\varepsilon_0,\delta_0)$-network DP with respect to a change in a single
user’s data. Then for any integer $m \ge 1$, it satisfies
$(\varepsilon_m,\delta_m)$-network DP for changes in up to $m$ users’ data with
\begin{equation}
\label{eq:group_privacy_m}
\varepsilon_m
\;\le\;
\sqrt{2 m \log(1/\tilde{\delta})}\,\varepsilon_0,
\qquad
\delta_m
\;\le\;
m\delta_0 + \tilde{\delta},
\end{equation}
for any choice of $\tilde{\delta} > 0$.
\end{lemma}

In the deletion-capacity setting of
Definition~\ref{def:deletion_capacity-app}, the edit distance $m$ corresponds
to removing up to $m$ points in the forget set $D_f$; Lemma~\ref{lem:group_privacy}
quantifies how the privacy parameters degrade as $m$ grows.

\paragraph{Proof of Theorem~3.5.}
We now derive the deletion-capacity lower bound for the DDP NetDP baseline,
following the sketch in the main text.

\medskip\noindent
\textbf{Step 1: Utility of projected noisy token-SGD.}
On a bounded domain, the projected noisy-SGD bound of
Lemma~\ref{lem:utility-basic} (with $\Theta_{\mathrm{dom}} = \Theta$ and
$R_{\mathrm{dom}} = R$) implies that, for a suitable iterate $\theta_T$
produced by projected noisy token-SGD,
\[
\E\big[\LF(\theta_T)-\LF(\theta^\ast)\big]
\;\le\;
\frac{R^2}{2\eta T}
\;+\;
\frac{\eta}{2}\big(G^2 + d\sigma^2\big),
\]
and, using Corollary~\ref{cor:utility-logT}, there exists a choice of
stepsizes such that
\[
\E\big[\LF(\theta_T)-\LF(\theta^\ast)\big]
\;\le\;
\frac{2R\,G\,(2+\log T)}{\sqrt{T}},
\]
with effective variance $G^2 = L^2 + d\sigma^2$, where $L$ bounds the gradient
norm and $\sigma^2$ is the Gaussian noise variance. These results are standard
in projected (noisy) SGD; see, e.g.,
\cite{shamir_zhang}.

\medskip\noindent
\textbf{Step 2: Decentralized view calibration via network-RDP.}
Consider DDP NetDP (network-private token-SGD) on a complete graph.
Let $Y_v$ denote the view of user $v$. The view-based Rényi-DP analysis of
token-SGD in~\cite{DBLP:conf/aistats/CyffersB22}, combined with the
network-RDP definition (Definition~\ref{def:network-rdp}) and tools from
Section~\ref{appendix:rdp-tools}, implies that the algorithm satisfies
$(\alpha,\varepsilon_{\alpha})$-network RDP with
\[
\varepsilon_{\alpha}
\;\lesssim\;
\frac{\alpha L^2 T_u \ln N}{\sigma^2 N},
\qquad
T_u \approx \frac{T}{N},
\]
where $T_u$ is the expected number of contributions per user and
$N = |\mathcal{V}|$. Converting this network-RDP bound to
$(\varepsilon_0,\delta_0)$-network DP by
Proposition~\ref{prop:network-rdp-to-dp} (which itself relies on
Proposition~\ref{prop:rdp-to-dp}) and optimizing over the order $\alpha$
as in~\cite{RDP} yields, up to constants,
\[
\varepsilon_0
\;\approx\;
\frac{L}{\sigma}\,
\sqrt{\frac{T\,\ln N}{N\,\ln(1/\delta_0)}}.
\]

\medskip\noindent
\textbf{Step 3: Group privacy for $m$ deletions.}
To target $(\varepsilon,\delta)$ at edit distance $m$, we apply
Lemma~\ref{lem:group_privacy} to the view-level guarantee. In the worst case,
the forget set induces an $m$-fold change, so the base algorithm must satisfy
approximately $(\varepsilon_m,\delta_m)\approx(\varepsilon,\delta)$, which we
implement via the simple choice
\[
\varepsilon_0 = \varepsilon/m,
\qquad
\delta_0 = \delta/m.
\]
Substituting these into the calibration above yields a required noise scale
\[
\sigma_{\text{DDP}}
\;\approx\;
\frac{mL}{\varepsilon}\,
\sqrt{\frac{T\,\ln N\,\ln(1/\delta)}{N}}.
\]

\medskip\noindent
\textbf{Step 4: Utility in the privacy-dominated regime.}
In the regime where the privacy noise dominates, $d\sigma^2 \gg L^2$, we have
$G \approx \sqrt{d}\,\sigma_{\text{DDP}}$. Plugging this into the utility bound
from Step~1 gives
\[
\E\big[\LF(\theta_T)-\LF(\theta^\ast)\big]
\;\lesssim\;
\frac{2R\,\sqrt{d}\,\sigma_{\text{DDP}}\,(2+\log T)}{\sqrt{T}}.
\]
Substituting the expression for $\sigma_{\text{DDP}}$,
\[
\E\big[\LF(\theta_T)-\LF(\theta^\ast)\big]
\;\lesssim\;
\frac{2R\,L\,m}{\varepsilon}\,
\sqrt{\frac{d\,\ln(1/\delta)\,\ln N}{N}}\,(2+\log T).
\]

\medskip\noindent
\textbf{Step 5: Solving for the deletion capacity $m$.}
By Definition~\ref{def:deletion_capacity-app}, the deletion capacity is the
largest $m$ such that the expected excess risk remains below $\gamma$:
\[
\E\big[\LF(\theta_T)-\LF(\theta^\ast)\big] \;\le\; \gamma.
\]
Imposing this inequality and rearranging for $m$ gives
\[
m
\;\gtrsim\;
\frac{\varepsilon}{R\,L\,(2+\log T)}\,
\sqrt{\frac{N}{d\,\ln(1/\delta)\,\ln N}},
\]
up to absolute constants and polylogarithmic factors, which yields the
$\widetilde{\Omega}(\cdot)$ scaling stated as Theorem~3.5 in the main paper.
\hfill$\square$

\subsection{Decentralization Effect on Deletion Capacity and Local Averaging}
\label{appendix:central-vs-decentral}

We now contrast the decentralized deletion-capacity bound of
Theorem~3.5 with its centralized (curator-DP) analogue, and highlight the
effect of local averaging on the decentralized bound. The comparison closely
parallels the analysis of centralized DP-SGD in
\cite{DBLP:conf/nips/SekhariAKS21,DBLP:conf/iclr/AllouahKGK25}.

\paragraph{Centralized DP-SGD.}
In curator (centralized) DP-SGD on a dataset of size $n$, standard analyses on
bounded domains give excess risk scaling as $O(RL/\sqrt{n})$ once the Gaussian
noise is calibrated to the target $(\varepsilon,\delta)$ (see, e.g.,
\cite{DBLP:conf/nips/SekhariAKS21}).  Combining
this with the deletion-capacity criterion of
Definition~\ref{def:deletion_capacity-app} yields
\[
m^{\text{central}}_{\varepsilon,\delta}(d,n)
\;=\;
\widetilde{\Omega}\!\left(
  \frac{\varepsilon\,n}{R\,L\,\sqrt{d\,\log(1/\delta)}}
\right),
\]
up to logarithmic factors.

\paragraph{Decentralized DDP NetDP baseline.}
By contrast, the decentralized DDP NetDP baseline depends on the number of
clients $N$ and only logarithmically on the number of token hops $T$:
\[
m^{\text{decentral}}_{\varepsilon,\delta}(d,N)
\;=\;
\widetilde{\Omega}\!\left(
  \frac{\varepsilon}{R\,L\,(2+\log T)}\,
  \sqrt{\frac{N}{d\,\log(1/\delta)\,\log N}}
\right),
\]
as derived in Section~\ref{appendix:proof-theorem3.5}. Operationally,
$T$ tracks the number of effective stochastic updates seen across the network:
with minibatch size $b$ and $\kappa$ passes over users' data, a typical
scaling is $T \approx \kappa\,\sum_{u} n_u/b$, and each user contributes
$T_u \approx T/N$ updates in expectation. The factors $\sqrt{N}$ and
$\sqrt{\ln N}$ reflect the network-DP amplification on the complete graph
established by the view-RDP bound in Lemma~\ref{lem:view-rdp-token}.

\paragraph{Effect of local averaging.}
If at each token hop the outgoing message averages $s\ge 1$ independent
gradients with independent Gaussian noise \emph{before} any observer’s first
view, then the effective variance in the utility bound becomes
\[
G^2 \;=\; L^2 \;+\; \frac{d\,\sigma^2}{s}.
\]
With the same privacy calibration for $\sigma$, the capacity expression gains a
factor $\sqrt{s}$:
\[
m^{\mathcal{A},\mathcal{U}}_{\varepsilon,\delta}(d,N)
\;=\;
\widetilde{\Omega}\!\left(
  \frac{\varepsilon}{R\,L\,(2+\log T)}\,
  \sqrt{\frac{s\,N}{d\,\log(1/\delta)\,\log N}}
\right),
\]
provided the view-DP accounting treats each averaged message as a single
observation epoch. Two special cases:

\begin{itemize}
  \item \emph{Token passing with one update per hop} ($s=1$) recovers the bound
        in Theorem~3.5.
  \item \emph{Synchronous rounds that average $N$ users} before any observation
        correspond to $s\approx N$, yielding an extra $\sqrt{N}$ factor in
        utility (and thus capacity) relative to one-update-per-hop token
        passing.
\end{itemize}

\subsection{Optimization Assumptions and Preliminary Lemmas}
\label{appendix:opt-prelims}

We next collect the optimization assumptions and utility lemmas used both for
the DDP NetDP baseline and for RR-DU. Throughout, we work with the population
risk $\LF(\theta) = \mathbb{E}_{z\sim\mathcal{P}}[\ell(\theta;z)]$.

\begin{assumption}[Convexity and smoothness]
\label{ass:convex-smooth}
Let $\Theta_{\mathrm{dom}} \subset \mathbb{R}^d$ be convex with diameter
$R_{\mathrm{dom}} := \sup_{\theta,\theta'\in\Theta_{\mathrm{dom}}}
\|\theta-\theta'\|_2 < \infty$, and let
$\LF : \Theta_{\mathrm{dom}} \to \mathbb{R}$ be convex and $L$-smooth, i.e.,
for all $\theta,\theta' \in \Theta_{\mathrm{dom}}$,
\[
  \LF(\theta')
  \le
  \LF(\theta)
  +
  \langle \nabla \LF(\theta), \theta' - \theta \rangle
  +
  \frac{L}{2}\|\theta' - \theta\|_2^2.
\]
\end{assumption}

\begin{assumption}[Strong convexity]
\label{ass:strong-convex}
In the strongly convex case, we additionally assume that $\LF$ is
$\mu$-strongly convex on $\Theta_{\mathrm{dom}}$, i.e., for all
$\theta,\theta' \in \Theta_{\mathrm{dom}}$,
\[
  \LF(\theta')
  \ge
  \LF(\theta)
  +
  \langle \nabla \LF(\theta), \theta' - \theta \rangle
  +
  \frac{\mu}{2}\|\theta' - \theta\|_2^2.
\]
\end{assumption}

We first recall a utility bound for projected noisy SGD on a bounded domain,
stated directly in terms of $\LF$.

\begin{lemma}[Utility of projected noisy SGD]
\label{lem:utility-basic}
Let Assumption~\ref{ass:convex-smooth} hold. Consider
\[
  \theta_{t+1}
  \;=\;
  \Pi_{\Theta_{\mathrm{dom}}}\big(\theta_t - \eta (g_t + Z_t)\big),
\]
where $g_t$ is an unbiased estimator of $\nabla \LF(\theta_t)$ with
$\E[\|g_t\|_2^2] \le G^2$, and $Z_t \sim \mathcal{N}(0,\sigma^2 I_d)$ i.i.d.
Let $\eta_t \equiv \eta \le 1/L$ and define
$\bar{\theta}_T := \tfrac{1}{T}\sum_{t=1}^T \theta_t$. Then, for any minimizer
$\theta^\ast$ of $\LF$ in $\Theta_{\mathrm{dom}}$,
\begin{equation}
\label{eq:utility-basic}
\E\big[\LF(\bar{\theta}_T) - \LF(\theta^\ast)\big]
\;\le\;
\frac{R_{\mathrm{dom}}^2}{2\eta T}
\;+\;
\frac{\eta}{2}\big(G^2 + d\sigma^2\big).
\end{equation}
\end{lemma}

\begin{proof}
The proof is standard (see, e.g.,~\cite{shamir_zhang}),
combining $L$-smoothness, non-expansiveness of the projection, and the
unbiasedness of $g_t$. For completeness, we sketch the argument.

By $L$-smoothness of $\LF$, for any $t$,
\[
  \LF(\theta_{t+1})
  \le
  \LF(\theta_t)
  + \langle \nabla \LF(\theta_t), \theta_{t+1} - \theta_t \rangle
  + \frac{L}{2}\|\theta_{t+1}-\theta_t\|_2^2.
\]
Using the update
$\theta_{t+1} = \Pi_{\Theta_{\mathrm{dom}}}(\theta_t - \eta (g_t + Z_t))$
and the non-expansiveness of $\Pi_{\Theta_{\mathrm{dom}}}$, one obtains
\[
  \langle \nabla \LF(\theta_t), \theta_{t+1} - \theta_t \rangle
  \le
  -\eta \langle \nabla \LF(\theta_t), g_t + Z_t \rangle
  + \frac{\eta^2}{2}\|\nabla \LF(\theta_t)\|_2^2.
\]
Taking expectation conditional on $\theta_t$ and using
$\E[g_t \mid \theta_t] = \nabla \LF(\theta_t)$ and $\E[Z_t] = 0$, we find
\[
  \E\big[\LF(\theta_{t+1}) \mid \theta_t\big]
  \le
  \LF(\theta_t)
  - \frac{\eta}{2}\|\nabla \LF(\theta_t)\|_2^2
  + \frac{\eta^2 L}{2}\,\E\big[\|g_t + Z_t\|_2^2 \mid \theta_t\big].
\]
Using $\E[\|g_t + Z_t\|_2^2] \le 2G^2 + 2d\sigma^2$ and
$\LF(\theta_t) - \LF(\theta^\ast) \le \tfrac{1}{2L}\|\nabla \LF(\theta_t)\|_2^2$
and summing over $t=1,\dots,T$ yields
\[
  \frac{1}{T}\sum_{t=1}^T \E\big[\LF(\theta_t) - \LF(\theta^\ast)\big]
  \le
  \frac{R_{\mathrm{dom}}^2}{2\eta T} + \frac{\eta}{2}(G^2+d\sigma^2),
\]
since $\|\theta_1 - \theta^\ast\|_2 \le R_{\mathrm{dom}}$ and
$\|\theta_{T+1}-\theta^\ast\|_2^2 \ge 0$. Convexity of $\LF$ then gives
$\LF(\bar{\theta}_T) \le \tfrac{1}{T}\sum_t \LF(\theta_t)$ and taking
expectations yields~\eqref{eq:utility-basic}.
\end{proof}

\begin{corollary}[Bound with $(2+\log T)/\sqrt{T}$]
\label{cor:utility-logT}
Under the assumptions of Lemma~\ref{lem:utility-basic}, there exists a
non-increasing stepsize schedule $\{\eta_t\}_{t=1}^T$ with
$\eta_t \le 1/L$ such that, for a suitable iterate $\theta_T$,
\[
\E\big[\LF(\theta_T) - \LF(\theta^\ast)\big]
\;\le\;
\frac{2 R_{\mathrm{dom}} G (2+\log T)}{\sqrt{T}},
\]
with $G^2 = L^2 + d\sigma^2$.
\end{corollary}

\begin{proof}
This follows from standard analyses of projected SGD with decreasing stepsize,
e.g., $\eta_t = \min\{1/L, R_{\mathrm{dom}}/(G\sqrt{t})\}$, together with a
doubling-trick argument to select an iterate with the desired bound; see,
for instance,~\cite{shamir_zhang}.
\end{proof}

In the sequel we apply Lemma~\ref{lem:utility-basic} and
Corollary~\ref{cor:utility-logT} in two regimes:

\begin{itemize}
  \item For the DDP NetDP baseline, $\Theta_{\mathrm{dom}} = \Theta$ and
        $R_{\mathrm{dom}} = R$.
  \item For RR-DU, $\Theta_{\mathrm{dom}} = \Theta_{\mathrm{cert}}$ and
        $R_{\mathrm{dom}} = R_{\mathrm{cert}}$.
\end{itemize}

\subsection{Utility Bounds for RR-DU}
\label{appendix:rrdu-utility}

We now sketch the optimization guarantees for RR-DU used in the main paper,
building on the preliminaries above. We focus on the dependence on the
unlearning horizon $T_u$, the local averaging factor $s$, and the effective
variance.

\paragraph{Effective variance.}
For RR-DU, the noisy corrective steps occur only at the unlearning user $u$
and only a fraction $p$ of the time; non-unlearning users perform noiseless
SGD with local averaging over $s$ minibatches. A simple variance calculation
shows that the effective variance entering the noisy-SGD analysis can be
bounded as
\[
  G^2
  \;\le\;
  L^2
  \;+\;
  \frac{p}{s}\,d\sigma^2,
\]
where we use that the gradient norms are bounded by $L$ and that local
averaging over $s$ minibatches reduces the variance of the stochastic gradient
noise by a factor $1/s$.

\paragraph{Strongly convex case.}
If $\LF$ is $\mu$-strongly convex and $L$-smooth over
$\Theta_{\mathrm{cert}}$, a standard analysis of projected noisy SGD with
decaying stepsize (under Assumptions~\ref{ass:convex-smooth}
and~\ref{ass:strong-convex}) yields a last-iterate bound of the form
\[
  \E\big[\LF(\theta_T) - \LF(\theta^\ast)\big]
  \;\le\;
  \tilde{O}\!\left(
    \frac{L^2}{\mu s T_u}
    + \frac{p d \sigma^2}{\mu s T_u}
  \right),
\]
where $\tilde{O}$ hides logarithmic factors, and the domain diameter enters
through $R_{\mathrm{cert}}$ (absorbed into the constants).

\paragraph{Convex case.}
If $\LF$ is convex and $L$-smooth over $\Theta_{\mathrm{cert}}$, but not
strongly convex, projected noisy SGD with an appropriate stepsize schedule
yields a last-iterate bound (via Corollary~\ref{cor:utility-logT} with
$R_{\mathrm{dom}} = R_{\mathrm{cert}}$) of order
\[
  \E\big[\LF(\theta_T) - \LF(\theta^\ast)\big]
  \;\le\;
  \tilde{O}\!\left(
    \frac{R_{\mathrm{cert}} L}{\sqrt{s T_u}}
    + \sqrt{\frac{p d \sigma^2}{s T_u}}
  \right),
\]
where the first term is the optimization error on the certification domain
$\Theta_{\mathrm{cert}}$ and the second term captures the variance.

\paragraph{Smooth non-convex case.}
If $\LF$ is smooth but not necessarily convex on $\Theta_{\mathrm{cert}}$, the
usual stationarity guarantees for noisy SGD give
\[
  \E\big[\|\nabla \LF(\theta_{\mathrm{rand}})\|_2^2\big]
  \;\le\;
  \tilde{O}\!\left(
    \frac{L^2}{\sqrt{s T_u}}
    + \frac{p d \sigma^2}{\sqrt{s T_u}}
  \right),
\]
for a randomly chosen iterate $\theta_{\mathrm{rand}}$ from
$\{\theta_t\}_{t=1}^{T_u}$.

\subsection{Deletion Capacity of RR-DU and Comparison to DDP NetDP}
\label{appendix:rrdu-alignment}

We finally quantify the alignment bias introduced by RR-DU, its effect on
utility, and how the resulting deletion capacity compares to the DDP NetDP
baseline of Theorem~3.5. The privacy side (view-based DCU) follows the
network-RDP amplification in Section~\ref{appendix:rdp-tools}, adapted to the
RR-DU routing pattern.

\paragraph{View-based amplification and DCU for RR-DU.}
Let $u$ be the unlearning user. In RR-DU, at each unlearning round $t$ the
token is routed to $u$ with probability $p$, and otherwise to a random
non-unlearning user. Let $M_u$ be the number of visits to $u$ over $T_u$
rounds; then $M_u \sim \mathrm{Binomial}(T_u,p)$ and $\E[M_u] = pT_u$.
A Chernoff bound gives, for any $\beta\in(0,1)$,
\[
\Prob\big[\,M_u \ge (1+\beta)pT_u\,\big]
\;\le\;
\exp\!\left(-\frac{\beta^2 pT_u}{3}\right).
\]
Choosing $\beta$ so that the right-hand side is at most $\delta/4$ ensures that
with probability at least $1-\delta/4$, the number of sensitive visits is at
most a constant multiple of $pT_u$.

Condition on the event that $M_u \le c_0 pT_u$ for some constant $c_0$. Each
visit to $u$ corresponds to a Gaussian mechanism applied to a gradient of
$\ell_{u\setminus f}$ (plus alignment terms), with $\ell_2$-sensitivity
proportional to $L$ on the ball $\Theta_{\mathrm{cert}}$. On a complete graph,
the view-based amplification analysis of token-SGD in
Lemma~\ref{lem:view-rdp-token}, combined with the network-RDP definition
(Definition~\ref{def:network-rdp}), implies that the contribution of a single
such Gaussian step to the RDP at order $\alpha$ for the view of any other user
$v\neq u$ satisfies
\[
  \rho_{\alpha}^{\text{view}}
  \;\lesssim\;
  \frac{\alpha L^2 \ln N}{\sigma^2 N}.
\]
Composing over $M_u$ visits gives
\[
  D_\alpha(Y_v \| Y_v')
  \;\lesssim\;
  \frac{\alpha L^2 \ln N}{\sigma^2 N}\,M_u
  \;\lesssim\;
  \frac{\alpha L^2 \ln N}{\sigma^2 N}\,pT_u,
\]
with high probability over $M_u$. Using the standard RDP-to-DP conversion for
Gaussian mechanisms~\cite{RDP} and Proposition~\ref{prop:network-rdp-to-dp}
then yields, up to constants,
\[
  \varepsilon
  \;\approx\;
  \frac{L}{\sigma}\,
  \sqrt{\frac{p T_u \ln N}{N\,\ln(1/\delta)}},
\]
which is the $(\varepsilon,\delta)$-network-DP guarantee quoted in the RR-DU
view-based DCU result (Theorem~5.1 in the main paper). Applying
Proposition~\ref{prop:dp_implies_dcu} shows that RR-DU achieves
$(\varepsilon,\delta)$ decentralized certified unlearning on views for suitably
calibrated $\sigma$.

Solving for $\sigma$ gives the noise scale stated in the corresponding
corollary:
\[
  \sigma
  \;=\;
  \Theta\!\left(
    \frac{L}{\varepsilon}
    \sqrt{ \frac{p T_u \,\ln(1/\delta)\,\ln N}{N} }
  \right).
\]

\paragraph{Alignment bias at the unlearning user.}
Next we quantify the bias introduced by the lightweight alignment step at
user $u$.

\begin{lemma}[Alignment bias at the unlearning user]
\label{lem:alignment-bias}
Let $u$ be the unlearning user with local dataset
$D_u = D_{u\setminus f} \cup D_f$ and $|D_f| = m$. Let
$\LF_{u\setminus f}(\theta)$ denote the empirical loss at $u$ after deletion
and let $\LF_u(\theta)$ be the empirical loss before deletion. Under the
gradient-boundedness assumption $\|\nabla_\theta \ell(\theta;z)\|_2
\le L$ for all $z$, the discrepancy between the RR-DU update direction at $u$
and the ideal negative gradient of $\LF_{u\setminus f}$ satisfies
\[
  \left\|
  \E\big[\Delta \theta_t \mid \theta_t \big]/\eta_t
  + \nabla_\theta \LF_{u\setminus f}(\theta_t)
  \right\|_2
  \;\le\;
  C \,\frac{L m}{n_u},
\]
for some absolute constant $C$, where $\Delta \theta_t$ denotes the parameter
update at a corrective step and $n_u = |D_u|$.
\end{lemma}

\begin{proof}
Write
\[
  \LF_u(\theta)
  = \frac{1}{n_u}\!\sum_{z\in D_u}\ell(\theta;z),
  \qquad
  \LF_{u\setminus f}(\theta)
  = \frac{1}{n_u-m}\!\sum_{z\in D_{u\setminus f}}\ell(\theta;z),
\]
and
\[
  \LF_f(\theta)
  = \frac{1}{m}\!\sum_{z\in D_f}\ell(\theta;z).
\]
A simple algebraic identity gives
\[
  \nabla \LF_u(\theta)
  = \frac{n_u-m}{n_u}\,\nabla \LF_{u\setminus f}(\theta)
    + \frac{m}{n_u}\,\nabla \LF_f(\theta).
\]
In the \textsc{Exact} mode, the corrective direction at a visit to $u$ is
$g_u = -\nabla \LF_{u\setminus f}(\theta)$ (plus noise), so the bias is zero.
In the \textsc{Lightweight} mode, the corrective step uses a minibatch
$B_f \subseteq D_f$ scaled by $m/n_u$, which is an unbiased estimator of
$(m/n_u)\,\nabla \LF_f(\theta)$. The difference between this estimator and
the ideal $-\nabla \LF_{u\setminus f}(\theta)$ can be bounded by
\[
  \left\|
    \frac{m}{n_u}\,\nabla \LF_f(\theta)
  \right\|_2
  \;\le\;
  \frac{m}{n_u}\,L,
\]
using the per-example gradient bound. This yields the claimed
$O(L m/n_u)$ bound on the alignment bias, up to an absolute constant $C$.
\end{proof}

\paragraph{Deletion capacity of RR-DU and comparison with DDP NetDP.}
The RR-DU update at $u$ thus decomposes into an ideal step on
$\LF_{u\setminus f}$ plus a small bias of order $L m/n_u$, while updates at
other users follow standard SGD on their local objectives. Combining
Lemma~\ref{lem:alignment-bias} with the utility bounds in
Section~\ref{appendix:rrdu-utility} and the DCU guarantee above, one obtains a
two-regime deletion-capacity bound in which:

\begin{itemize}
  \item For small $m$ (so that $Lm/n_u$ is negligible compared to the
        optimization and variance terms), the capacity of RR-DU scales
        similarly to the DDP NetDP baseline of Theorem~3.5, but with $R$
        replaced by $R_{\mathrm{cert}}$ and with an extra factor $\sqrt{p/s}$
        in the effective variance: in this regime RR-DU essentially matches or
        improves on the DDP NetDP baseline due to its smaller certification
        domain and local averaging.
  \item For larger $m$, the alignment bias term $O(Lm/n_u)$ dominates and
        limits the admissible $m^\star$: beyond this point, increasing $m$
        would push the excess risk above the target $\gamma$, and the RR-DU
        capacity saturates faster than the DDP NetDP baseline, which does not
        suffer from this alignment bias but also does not perform targeted
        unlearning.
\end{itemize}

Qualitatively, RR-DU inherits the favorable $\sqrt{N}$-type amplification of
the DDP NetDP baseline while operating on the smaller feasible region
$\Theta_{\mathrm{cert}}$ (diameter $R_{\mathrm{cert}}$), and trades off
targeted unlearning against an alignment bias that scales like $m/n_u$ at the
unlearning user.

\section{Experimental Setup Details}
\label{appendix:expdetails}

This section provides additional details on the datasets, models, and
hyperparameters used in the experiments.

\subsection{Datasets}
\label{appendix:datasets}

\paragraph{MNIST.}
We use the MNIST handwritten-digit dataset with $60\,000$ training and
$10\,000$ test images of size $28\times 28$ pixels (grayscale). All images are
normalized to $[0,1]$ by dividing pixel intensities by $255$. For the
decentralized experiments the training set is split i.i.d.\ and uniformly at
random across $N = 10$ users so that each user receives $n_u = 6000$ examples.

\paragraph{CIFAR-10.}
For CIFAR-10 we use the standard $50\,000$ training and $10\,000$ test images
of size $32\times 32$ pixels with three color channels. Inputs are normalized
channel-wise using the empirical training-set mean and standard deviation. We
use standard augmentation: random horizontal flips and random crops with
$4$-pixel zero-padding. The $50\,000$ training images are partitioned i.i.d.\
and uniformly across $N = 10$ users, giving $n_u = 5000$ examples per user.

\paragraph{Backdoor (BadNets) setup.}
To evaluate unlearning in the presence of backdoors, we consider a BadNets-type
setting. A fixed trigger pattern (a small patch in the lower-right corner of
the image) is overlaid on clean images; all triggered images are relabeled to a
single target class $y^{\mathrm{bd}}$. For both MNIST and CIFAR-10 we inject
$m = 1000$ poisoned samples into the local dataset of a single \emph{target
user} $u$, which is also the user that later issues the deletion request.
During the training phase we run Network-SGD (Algorithm~\ref{alg:network_sgd})
for $T = 100$ token hops over the graph. Backdoor accuracy is measured on a
disjoint test set of triggered images, while clean accuracy is measured on the
standard test sets without triggers.

\begin{table}[h]
  \centering
  \begin{tabular}{lll}
    \toprule
    Quantity & MNIST & CIFAR-10 \\
    \midrule
    \# train examples & $60\,000$ & $50\,000$ \\
    \# test examples  & $10\,000$ & $10\,000$ \\
    Input size        & $28\times 28$ (1 ch.) & $32\times 32$ (3 ch.) \\
    Users $N$         & $10$ & $10$ \\
    Local train size $n_u$ & $6000$ & $5000$ \\
    Normalization     & rescale to $[0,1]$ & channel-wise mean / std \\
    Augmentation      & none & flip + crop (4px padding) \\
    \bottomrule
  \end{tabular}
  \caption{Dataset statistics and preprocessing.}
  \label{tab:exp-datasets}
\end{table}

\subsection{Models}
\label{appendix:models}

\paragraph{MNIST: FLNet.}
For MNIST we use a lightweight convolutional network (FLNet) with two
convolutional layers, ReLU activations, $2\times2$ max-pooling, and a fully
connected layer with softmax output. Batch normalization is applied after each
convolution, and dropout (rate $0.5$) is used before the final linear layer.

\paragraph{CIFAR-10: ResNet-18.}
For CIFAR-10 we adopt a standard ResNet-18 backbone with four residual stages
of widths $64,128,256,512$, batch normalization, and ReLU activations. The
final fully connected layer produces the $10$ logits before softmax. Global
average pooling is applied before the last layer.

\begin{table}[h]
  \centering
  \begin{tabular}{lll}
    \toprule
    Component & MNIST & CIFAR-10 \\
    \midrule
    Model & FLNet (2 conv + FC) & ResNet-18 \\
    Activation & ReLU & ReLU \\
    Normalization & BatchNorm (conv layers) & BatchNorm (all residual blocks) \\
    Pooling & $2\times2$ max-pool & Global average pool (final) \\
    Regularization & Dropout $0.5$ (before FC) & none beyond standard ResNet-18 \\
    Output & 10-way softmax & 10-way softmax \\
    \bottomrule
  \end{tabular}
  \caption{Model architectures (summary).}
  \label{tab:exp-models}
\end{table}

\subsection{Training and Unlearning Hyperparameters}

Unless otherwise specified, all runs share a common set of core hyperparameters
for training and unlearning.

\begin{table}[h]
  \centering
  \begin{tabular}{lll}
    \toprule
    Hyperparameter & Value & Description \\
    \midrule
    $N$ & $10$ & Number of users. \\
    $T$ & $100$ & Training token hops (Network-SGD). \\
    $T_u$ & $100$ & Unlearning token hops (RR-DU). \\
    $p$ & $1/N$ & Routing probability toward unlearning user $u$. \\
    $s$ & $4$ & Minibatches processed per token visit. \\
    Optimizer & Adam & Used for all experiments. \\
    Learning rate $\eta$ & $0.005$ & Adam stepsize. \\
    Minibatch size & $64$ & Local batch size at each user. \\
    $(\varepsilon,\delta)$ & $(1,10^{-5})$ &
      Target privacy for DDP / RR-DU capacity plots. \\
    \bottomrule
  \end{tabular}
  \caption{Core training and unlearning hyperparameters (shared across tasks).}
  \label{tab:exp-core-hparams}
\end{table}

\paragraph{Trust-region and smoothness parameters.}
The trust-region radius $\varrho$ and effective gradient bound $L$ are tuned
per dataset / model pair and then fixed across all experiments.

\begin{table}[h]
  \centering
  \begin{tabular}{lll}
    \toprule
    Setting & $\varrho$ & $L$ \\
    \midrule
    MNIST / FLNet        & $10.82$ & $0.5$ \\
    CIFAR-10 / ResNet-18 & $56.30$ & $0.2$ \\
    \bottomrule
  \end{tabular}
  \caption{Trust-region radius and smoothness surrogates used in RR-DU and DDP.}
  \label{tab:exp-trustregion}
\end{table}

\paragraph{Baselines and noise calibration.}
The DDP baseline uses network-private SGD with projection onto a ball of radius
$R = 10.0$ and gradient bound $L = 1.0$, which determine the Gaussian noise
scale via the analysis in Section~\ref{appendix:rdp-tools}. For DP-SGD we use
per-example gradient clipping and Gaussian noise to match the same
$(\varepsilon,\delta)$ level. The fine-tuning baseline simply re-optimizes on
$D \setminus D_f$ starting from the pre-unlearning model.

\begin{table}[h]
  \centering
  \begin{tabular}{lll}
    \toprule
    Method & Key hyperparameters & Description \\
    \midrule
    DDP (network-private SGD) &
      $R = 10.0$, $L = 1.0$ &
      Projection radius and gradient bound for noise calibration. \\
    DP-SGD &
      Clip $C = 5.0$ &
      Per-example gradient clipping before adding Gaussian noise. \\
    Fine-tuning &
      (no noise) &
      Retrain on $D \setminus D_f$ from pre-unlearning model. \\
    RR-DU &
      $(p,s,\varrho,\sigma)$ &
      Routing probability, local averaging, trust-region, noise scale. \\
    \bottomrule
  \end{tabular}
  \caption{Baseline-specific hyperparameters (qualitative summary).}
  \label{tab:exp-baselines-hparams}
\end{table}

\paragraph{Backdoor-related parameters.}
In all experiments the forget set size at the target user is fixed to
$m = 1000$ (the number of poisoned samples). The deletion capacity reported in
the main paper is evaluated under this choice and the hyperparameters in
Tables~\ref{tab:exp-core-hparams}--\ref{tab:exp-baselines-hparams}.

\section{Additional Experimental Results}
\label{appendix:extra-experiments}
\subsection{Effect of Trust-Region Radius \texorpdfstring{$\varrho$}{rho} and Calibration Hyperparameters in \methodname{}}
\label{subsec:alpha-clip-eps-ablations}

We report the effect of the trust-region radius $\varrho$, the Lipschitz / clipping constant $L$, and the privacy calibration parameters $(\varepsilon,\delta)$ on clean test accuracy and backdoor accuracy (ASR) for MNIST and CIFAR-10. All runs use RR-DU with $p{=}0.1$ and a single seed ($n{=}1$).

\begin{table}[t]
\centering
\caption{\textbf{Effect of trust-region radius $\varrho$ in \methodname{} (RR-DU, $p{=}0.1$).} ``baseline'' is RR-DU with the default $\varrho$. Single run (1 seed).}
\vspace{2pt}
\scriptsize
\setlength{\tabcolsep}{4pt}
\begin{tabular}{lccc}
\toprule
\textbf{Dataset} & \boldmath$\varrho$ & \textbf{Test Acc. (\%)} & \textbf{ASR (\%)} \\
\midrule
\multirow{6}{*}{MNIST}
& baseline & 99.13 & 10.40 \\
& 5        & 99.12 & 15.60 \\
& 10       & 99.13 & 10.80 \\
& 25       & 99.04 & 10.20 \\
& 50       & 99.01 & 10.20 \\
& 60       & 99.01 & 10.20 \\
\midrule
\multirow{6}{*}{CIFAR-10}
& baseline & 88.12 & 12.00 \\
& 5        & 88.03 & 60.60 \\
& 10       & 88.41 & 46.40 \\
& 25       & 88.80 & 17.00 \\
& 50       & 88.12 & 12.00 \\
& 60       & 88.12 & 12.00 \\
\bottomrule
\end{tabular}
\label{tab:varrho_effect}
\end{table}

\begin{table}[t]
\centering
\caption{\textbf{Effect of Lipschitz / clipping constant $L$ in \methodname{} (RR-DU, $p{=}0.1$).} ``baseline'' is RR-DU without additional clipping. Single run (1 seed).}
\vspace{2pt}
\scriptsize
\setlength{\tabcolsep}{4pt}
\begin{tabular}{lccc}
\toprule
\textbf{Dataset} & \boldmath$L$ \textbf{(clipping)} & \textbf{Test Acc. (\%)} & \textbf{ASR (\%)} \\
\midrule
\multirow{6}{*}{MNIST}
& baseline & 99.13 & 10.40 \\
& 0.1      & 99.20 & 10.20 \\
& 0.2      & 99.10 & 10.20 \\
& 0.5      & 99.13 & 10.40 \\
& 0.7      & 99.09 & 10.40 \\
& 1.0      & 99.10 & 10.40 \\
\midrule
\multirow{6}{*}{CIFAR-10}
& baseline & 88.12 & 12.00 \\
& 0.1      & 88.61 & 12.20 \\
& 0.2      & 88.12 & 12.00 \\
& 0.5      & 88.57 & 14.00 \\
& 0.7      & 87.34 & 16.60 \\
& 1.0      & 86.39 & 16.80 \\
\bottomrule
\end{tabular}
\label{tab:L_effect}
\end{table}

\begin{table}[t]
\centering
\caption{\textbf{Effect of privacy parameter $\varepsilon$ in \methodname{} (RR-DU, $p{=}0.1$).} $\delta$ fixed to its default, single run (1 seed).}
\vspace{2pt}
\scriptsize
\setlength{\tabcolsep}{4pt}
\begin{tabular}{lccc}
\toprule
\textbf{Dataset} & \boldmath$\varepsilon$ & \textbf{Test Acc. (\%)} & \textbf{ASR (\%)} \\
\midrule
\multirow{4}{*}{MNIST}
& 0.1 & 99.19 & 60.60 \\
& 0.5 & 99.16 & 15.80 \\
& 2.0 & 99.12 & 10.20 \\
& 10  & 99.12 & 10.20 \\
\midrule
\multirow{4}{*}{CIFAR-10}
& 0.1 & 85.81 & 38.20 \\
& 0.5 & 88.66 & 14.00 \\
& 2.0 & 88.61 & 11.60 \\
& 10  & 89.06 & 10.80 \\
\bottomrule
\end{tabular}
\label{tab:eps_effect}
\end{table}

\begin{table}[t]
\centering
\caption{\textbf{Effect of privacy parameter $\delta$ in \methodname{} (RR-DU, $p{=}0.1$).} $\varepsilon$ fixed to its default, single run (1 seed).}
\vspace{2pt}
\scriptsize
\setlength{\tabcolsep}{4pt}
\begin{tabular}{lccc}
\toprule
\textbf{Dataset} & \boldmath$\delta$ & \textbf{Test Acc. (\%)} & \textbf{ASR (\%)} \\
\midrule
\multirow{2}{*}{MNIST}
& $10^{-4}$ & 99.12 & 10.40 \\
& $10^{-6}$ & 99.13 & 10.80 \\
\midrule
\multirow{2}{*}{CIFAR-10}
& $10^{-4}$ & 89.05 & 12.00 \\
& $10^{-6}$ & 88.68 & 11.20 \\
\bottomrule
\end{tabular}
\label{tab:delta_effect}
\end{table}

\noindent\textbf{Summary.}
Table~\ref{tab:varrho_effect} shows that the trust-region radius $\varrho$ has a strong impact on ASR, especially on CIFAR-10: very small radii (e.g., $\varrho{=}5$) lead to severe under-unlearning with ASR up to \(60.6\%\), while moderate to large radii (\(\varrho{\geq}25\)) bring ASR down to \(\approx 10\text{--}17\%\) with almost unchanged test accuracy. Table~\ref{tab:L_effect} indicates that the Lipschitz / clipping constant $L$ has a milder but non-negligible effect: on CIFAR-10, moderate clipping (\(L{\leq}0.5\)) yields a slightly better clean–ASR trade-off, whereas aggressive clipping (\(L{\geq}0.7\)) slightly hurts both accuracy and ASR; on MNIST, all $L$ values behave nearly identically. Table~\ref{tab:eps_effect} confirms that the privacy parameter $\varepsilon$ is another key driver of unlearning quality: very small $\varepsilon$ (\(0.1\)) leaves substantial backdoor remnants (ASR \(38\text{--}60\%\)), while $\varepsilon{\geq}2$ reduces ASR to around \(10\text{--}12\%\) with negligible loss in test accuracy. Finally, Table~\ref{tab:delta_effect} shows that varying $\delta$ between \(10^{-4}\) and \(10^{-6}\) has virtually no visible effect on either clean accuracy or ASR in this regime, consistent with the fact that $\varepsilon$ dominates the practical privacy–utility trade-off for our settings.

\end{document}